%% file: main.tex
\documentclass[sigconf]{aamas}  
\pdfoutput=1

\AtBeginDocument{%
  \providecommand\BibTeX{{%
    \normalfont B\kern-0.5em{\scshape i\kern-0.25em b}\kern-0.8em\TeX}}}
    
\usepackage{booktabs}
\usepackage{color, colortbl}
\usepackage{caption,subcaption}
\usepackage{graphicx,wrapfig}
\usepackage{algorithm}
\usepackage{algpseudocode}
\usepackage{flushend} 
\setcopyright{ifaamas}  
\copyrightyear{2020} 
\acmYear{2020} 
\acmDOI{} 
\acmPrice{} 
\acmISBN{} 
\acmConference[AAMAS'20]{Proc.\@ of the 19th International Conference on Autonomous Agents and Multiagent Systems (AAMAS 2020)}{May 9--13, 2020}{Auckland, New Zealand}{B.~An, N.~Yorke-Smith, A.~El~Fallah~Seghrouchni, G.~Sukthankar (eds.)}  

\input{math_commands.tex}

\newcommand{\pmin}[1] {
  {\scriptsize $\pm$ #1}}

\definecolor{Gray}{gray}{0.9}

\begin{document}
\title{Plannable Approximations to MDP Homomorphisms: Equivariance under Actions}



\author{Elise van der Pol}
\affiliation{%
	\institution{UvA-Bosch Deltalab, University of Amsterdam}
}
	\email{e.e.vanderpol@uva.nl}
\author{Thomas Kipf}
\affiliation{%
  \institution{University of Amsterdam}
}
  \email{t.n.kipf@uva.nl}
\thanks{TK is now at Google Research, Brain Team}
\author{Frans A. Oliehoek}
\affiliation{%
	\institution{Delft University of Technology}
}
        \email{f.a.oliehoek@tudelft.nl}
\author{Max Welling}
\affiliation{%
	\institution{UvA-Bosch Deltalab, University of Amsterdam}
}
        \email{m.welling@uva.nl}

\begin{abstract}  
\input{abstract}
\end{abstract}

%


\maketitle


\section{Introduction}
\input{introduction}

\section{Background}
\input{background}

\section{Learning MDP Homomorphisms}
\input{method}

\section{Experiments} \label{section:experiments}
\input{experiments}

\section{Related Work} \label{section:related}
\input{related}

\section{Conclusion}
\input{conclusion}

\section{Acknowledgments}
\input{acknowledgments}


\bibliographystyle{ACM-Reference-Format}  
\bibliography{sample-bibliography}  

\appendix

\end{document}

%% file: math_commands.tex
\def\S{{\mathcal{S}}}
\def\A{{\mathcal{A}}}
\def\s{{s}}
\def\a{{a}}
\def\trans{{T}}
\def\R{{R}}
\def\T{{T}}

\def\M{{\mathcal{M}}}
\def\aM{{\mathcal{\bar{M}}}}
\def\daM{{\mathcal{\hat{M}}}}

\def\aS{{\mathcal{Z}}}
\def\aA{{\mathcal{\bar{A}}}}

\def\as{{z}}

\def\proj{{Z}}
\def\projparam{{Z_\theta}}

\def\atrans{{\bar{T}}}
\def\atransparam{{\bar{T}_\phi}}
\def\adeltaparam{{\bar{A}_\phi}}
\def\adelta{{\bar{A}_s}}

\def\aR{{\bar{R}}}
\def\aRparam{{\bar{R}_\zeta}}
\def\aa{{\bar{a}}}

\def\daS{{\mathcal{X}}}
\def\daA{{\mathcal{\hat{A}}}}
\def\das{{x}}
\def\daa{{\hat{a}}}

\def\datrans{{\hat{T}_{\phi}}}
\def\predatrans{{\hat{T}_{\phi}'}}
\def\daR{{\hat{R}_{\zeta}}}

\def\tautrans{{\tau}}

\newcommand{\amax}[0]{\underset{a}{\text{ max }}}

\def\hom{{h}}

\def\dapi{{\hat{\pi}}}

%% file: abstract.tex
This work exploits action equivariance for representation learning in reinforcement learning. Equivariance under actions states that transitions in the input space are mirrored by equivalent transitions in latent space, while the map and transition functions should also commute. We introduce a contrastive loss function that enforces action equivariance on the learned representations. We prove that when our loss is zero, we have a homomorphism of a deterministic Markov Decision Process (MDP). Learning equivariant maps leads to structured latent spaces, allowing us to build a model on which we plan through value iteration. We show experimentally that for deterministic MDPs, the optimal policy in the abstract MDP can be successfully lifted to the original MDP. Moreover, the approach easily adapts to changes in the goal states. Empirically, we show that in such MDPs, we obtain better representations in fewer epochs compared to representation learning approaches using reconstructions, while generalizing better to new goals than model-free approaches.

%% file: introduction.tex
Dealing with high dimensional state spaces and unknown environmental dynamics presents an open problem in decision making~\cite{henderson2017deep}. Classical dynamic programming approaches require knowledge of environmental dynamics and low dimensional, tabular state spaces~\cite{puterman1994markov}. Recent deep reinforcement learning methods on the other hand offer good performance, but often at the cost of being unstable and sample-hungry~\cite{henderson2017deep, mnih2016asynchronous, jaderberg2016reinforcement, mnih2015human}. The deep model-based reinforcement learning literature aims to fill this gap, for example by finding policies after learning models based on input reconstruction~\cite{kurutach2018model,ha2018world, zhang2018solar, corneil2018efficient}, by using environmental models in auxiliary losses~\cite{farquhar2018treeqn, jaderberg2016reinforcement}, or by forcing network architectures to resemble planning algorithms~\cite{tamar2016value, oh2017value}. While effective in learning end-to-end policies, these types of approaches are not forced to learn good representations and may thus not build proper environmental models. In this work, we focus on learning representations of the world that are suitable for exact planning methods. To combine dynamic programming with the representational power of deep networks, we factorize the online decision-making problem into a self-supervised model learning stage and a dynamic programming stage. 
\begin{wrapfigure}{r}{0.5\columnwidth}
    \includegraphics[width=0.5\columnwidth]{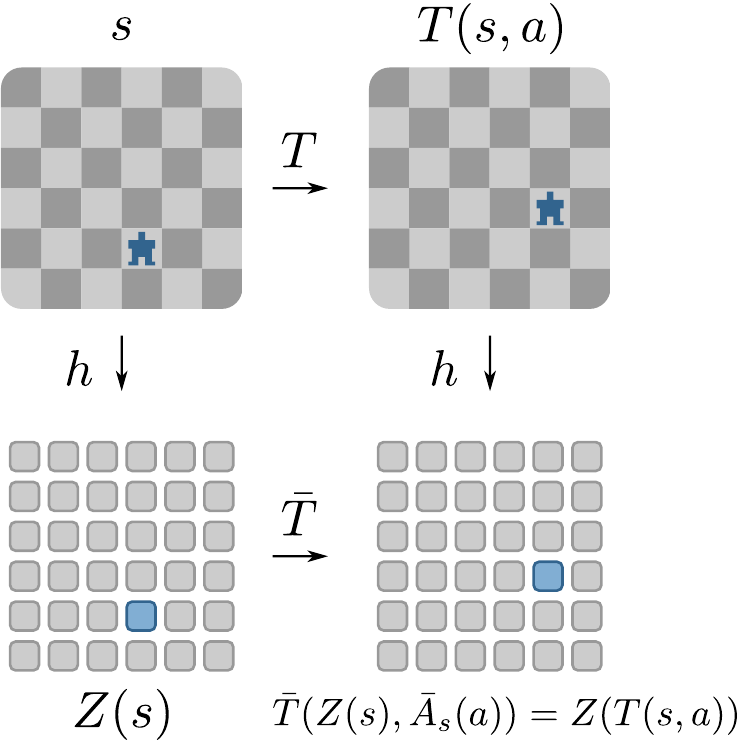}
     \caption{Visualization of the notion of equivariance under actions. We say $\proj$ is an action equivariant mapping if $\proj(\trans(\s, \a)) = \atrans(\proj(\s), \adelta(a))$.}
    \label{fig:action_equivariance}
\end{wrapfigure} We do this under the assumption that good representations minimize MDP metrics~\cite{givan2003equivalence, ferns2004metrics, li2006towards, taylor2009bounding}. While such metrics have desirable theoretical guarantees, they require an enumerable state space and knowledge of the environmental dynamics, and are thus not usable in many problems. To resolve this issue, we propose to learn representations using the more flexible notion of action equivariant mappings, where the effects of actions in input space are matched by equivalent action effects in the latent space. See Figure~\ref{fig:action_equivariance}. We make the following contributions. First, we propose learning an equivariant map and corresponding action embeddings. This corresponds to using MDP homomorphism~\cite{ravindran2004approximate} metrics~\cite{taylor2009bounding} of deterministic MDPs, enabling planning in the homomorphic image of the original MDP. Second, we prove that for deterministic MDPs, when our loss is zero, we have an MDP homomorphism. This means that the resulting policy can be lifted to the original MDP. 
Third, we provide experimental evaluation in a variety of settings to show 1) that we can recover the graph structure of the input MDP, 2) that planning in this abstract space results in good policies for the original space, 3) that we can change to arbitrary new goal states without further gradient descent updates and 4) that this works even when the input states are continuous, or when generalizing to new instances with the same dynamics. 


%% file: background.tex
\paragraph{Markov Decision Processes} An infinite horizon Markov Decision Process (MDP) is a tuple $\M = (\S, \A, \R, \T, \gamma)$, where $\s \in \S$ is a Markov state, $a \in \A$ is an action that an agent can take, $\R: \S\times\A\rightarrow \mathbb{R}$ is a reward function that returns a scalar signal $r$ defining the desirability of some observed transition, $0 \leq \gamma \leq 1$
is a discount factor that discounts future rewards exponentially and $\T: \S \times \A \times \S \rightarrow [0, 1]$ is a transition function, that for a pair of states and an action assigns a probability of transitioning from the first to the second state. The goal of an agent in an MDP is to find a policy $\pi: \S \times \A \rightarrow [0, 1]$, a function assigning probabilities to actions in states, that maximizes the \textit{return} $G_t = \sum_{k=0}^\infty \gamma^k r_{t+k+1}$. The expected return of a state, action pair under a policy $\pi$ is given by a Q-value function $Q_\pi: \S \times \A \rightarrow \mathbb{R}$ where $Q_\pi(s, a) = \mathbb{E}_\pi \left [ G_t | s_t=s, a_t=a \right ]$. The value of a state under an optimal policy $\pi^*$ is given by the value function $V^*: \S \rightarrow \mathbb{R}$, defined as $V^* = \max_a Q^*(s, a)$ under the Bellman optimality equation. 
\paragraph{Value Iteration} Value Iteration (VI) is a dynamic programming algorithm that finds Q-values in MDPs, by iteratively applying the Bellman optimality operator. This can be viewed as a graph diffusion where each state is a vertex and transition probabilities define weighted edges. VI is guaranteed to find the optimal policy in an MDP. For more details, see~\cite{puterman1994markov}.

\paragraph{Bisimulation Metrics}
To enable computing optimal policies in MDPs with very large or continuous state spaces, one approach is aggregating states based on their similarity in terms of environmental dynamics~\cite{dean1997model, li2006towards}. A key concept is the notion of \textit{stochastic bisimulations} for MDPs, which was first introduced by~\citet{dean1997model}. Stochastic bisimulation defines an equivalence relation on MDP states based on matching reward and transition functions, allowing states to be compared to each other. Later work~\cite{ferns2004metrics} observes that the notion of stochastic bisimulation is too stringent (everything must match exactly) and proposes using a more general \textit{bisimulation metric} instead, with the general form
\begin{align}
    d(s, s') &= \amax \Big (c_R |\R(s, a) - \R(s', a)| + c_T d_P(\T(s, a), \T(s', a)) \Big )
\end{align}
where $c_R$ and $c_T$ are weighting constants, $\T(\cdot, a)$ is a distribution over next states and $d_P$ is some probability metric, such as the Kantorovich (Wasserstein) metric. Such probability metrics are recursively computed. For more details, see~\citep{ferns2004metrics}. The bisimulation metric provides a distance between states that is not based on input features but on environmental dynamics.

\paragraph{MDP Homomorphism}
A generalization of the mapping induced by bisimulations is the notion of MDP homomorphisms~\cite{ravindran2004approximate}. MDP homomorphisms were introduced by~\cite{ravindran2001symmetries} as an extension of~\cite{dean1997model}. An MDP homomorphism $\hom$ is a tuple of functions $\left \langle \proj, \left \{ \adelta \right \} \right \rangle$ with $\proj: \S \rightarrow \aS$ a function that maps states to abstract states, and each $\adelta: \A \rightarrow \aA$ a state-dependent function that maps actions to abstract actions, that preserves the structure of the input MDP. We use the definition given by~\citet{ravindran2004approximate}: 
\begin{definition}[Stochastic MDP Homomorphism]A \emph{Stochastic MDP homomorphism }from
a stochastic MDP $\M=\left\langle \S, \A, \trans, \R\right\rangle $
to an MDP $\aM=\left\langle \aS, \aA, \atrans, \aR\right\rangle $
is a tuple $\hom=\left\langle \proj,\left\{ \adelta \right\} \right\rangle $
, with
\begin{itemize}
\item $\proj: \S\to\aS$ the state embedding function, and
\item $\adelta: \A\to\aA$ the action embedding functions,
\end{itemize}
such that the following identities hold:
\begin{align}
	\forall_{\s,\s'\in\S,\a\in\A}\quad \atrans(\proj(\s')|\proj(\s), \adelta(a)) &= \sum_{\s''\in[\s']_{\proj}}\trans(\s''|\s,\a) \label{eq:hom_trans} \\
	\forall_{\s\in\S,\a\in\A}\quad \aR(\proj(\s),\adelta(\a)) &= \R(\s,\a) \label{eq:hom_rew}
\end{align}
Here, $[\s']_{\proj}=\proj^{-1}(\proj(\s'))$ is the equivalence class of $\s'$
under $Z$.\end{definition}
%

We specifically consider deterministic MDPs. In that case:
\begin{definition}[Deterministic MDP Homomorphism]A \emph{Deterministic MDP homomorphism }from
a deterministic MDP $\M=\left\langle \S, \A,\T,\R\right\rangle $
to an MDP $\aM=\left\langle \aS, \aA, \atrans, \aR\right\rangle $
is a tuple $\hom=\left\langle \proj,\left\{ \adelta \right\} \right\rangle $
, with
\begin{itemize}
\item $\proj:\S\to\aS$ the state embedding function, and
\item $\adelta:\A\to\aA$ the action embedding functions,
\end{itemize}
such that the following identities hold:

\begin{align}
	\forall_{\s,\s'\in\S,\a\in\A}\quad \trans(\s,\a)=\s'\; \implies\; &\atrans(\proj(\s),\adelta(\a))=\proj(\s')  \label{eq:det_hom_trans} \\
	\forall_{\s\in\S,\a\in\A}\quad &\aR(\proj(\s),\adelta(\a))=\R(\s,\a) \label{eq:det_hom_rew}
\end{align}
\end{definition}
The states $\s$ are organized into equivalence classes under $\proj$ if they follow the same dynamics in $\as$-space. 
The MDP $\aM$ is referred to as the \textit{homomorphic image} of $\M$ under $h$~\cite{ravindran2004approximate}. An important property of MDP homomorphisms is that a policy optimal in homomorphic image $\aM$ can be \textit{lifted} to an optimal policy in $\M$~\cite{ravindran2004approximate, hartmanis1966algebraic}.  Looking at these definitions, it may be clear that MDP homomorphisms and bisimulation metrics are closely related. The difference is that the latter measures distances between two MDP states, while the former is a map from one MDP to another. However, the idea of forming a distance metric by taking a sum of the distances can be extended to homomorphisms, as proposed by~\citet{taylor2009bounding}:
\begin{multline}
    d((s, a), (\proj(s), \adelta(a) )) = 
        c_R |\R(s, a) - \aR( \proj(s), \adelta(a) )| 
        \\
        + c_T d_P (\proj \T(s, a), \atrans(\proj(s), \adelta(a))   ), 
\end{multline}
with $d_P$ a suitable measure of the difference between distributions (e.g., Kantorovich metric), and $\proj \T(s, a)$ shorthand for projecting the distribution over next states into the space of $\aS$ (see \cite{gelada2019deepmdp} for details).
We refer to this as the \emph{MDP homomorphism metric}. 
\paragraph{Action-Equivariance}
We define a mapping $Z:\S \rightarrow \aS$ to be action-equivariant if $Z(T(s, a)) = \bar{T}(Z(s), \bar{A}_s(a))$ and $R(s, a) = \bar{R}(Z(s), \bar{A}_s(a))$, i.e. when the constraints in Eq.~\ref{eq:det_hom_trans} and Eq.~\ref{eq:det_hom_rew} hold. \\

%% file: method.tex
We are interested in learning compact, plannable representations of MDPs. We call MDP representations \textit{plannable} if the optimal policy found by planning algorithms such as VI can be lifted to the original MDP and still be close to optimal. This is the case when the representation respects the original MDP's dynamics, such as when the equivariance constraints in Eq.~\ref{eq:det_hom_trans} and Eq.~\ref{eq:det_hom_rew} hold. 
In this paper we leverage MDP homomorphism metrics to find such representations. In particular, we introduce a loss function that enforces these equivariance constraints, then construct an abstract MDP in the learned representation space. We compute a policy in the abstract MDP $\aM$ using VI, and \textit{lift} the abstract policy to the original space. To keep things simple, we focus on deterministic MDPs, but in preliminary experiments our method performed well out of the box on stochastic MDPs. Additionally, the framework we outline here can be extended to the stochastic case, as~\citet{gelada2019deepmdp} does for bisimulation metrics. 

\subsection{Learning State Representations}
Here we show how to learn state representations that respect action-equivariance. We embed the states in $\S$ into Euclidean space using a contrastive loss based on MDP homomorphism metrics. Similar losses have often been used in related work~\cite{gelada2019deepmdp, cswm, anand2019unsupervised, oord2018representation, franccois2018combined}, which we compare in Section~\ref{section:related}. 
We represent the mapping $\proj$ using a neural network parameterized by $\theta$, whose output will be denoted $\projparam$. This function maps a state $s \in \S$ to a latent representation $\as \in \aS \subseteq \mathbb{R}^D$. We additionally approximate the abstract transition $\atrans$ by a function $\atransparam:\aS \times \aA \rightarrow \aS$ parameterized by $\phi$, and the abstract rewards $\aR$ by a neural network $\aRparam: \aS \rightarrow \mathbb{R}$, parameterized by $\zeta$, that predicts the reward for an abstract state. 
From Eq.~\ref{eq:det_hom_rew} we simplify to a state-dependent reward using $\R(\s) = \aR \left ( \proj(s) \right)  \label{eq:action_eq_rew}$ where $\R(s)$ is the reward function that outputs a scalar value for an $\s \in \S$, and $\aR$ is its equivalent in $\aM$.
During training,  we first sample a set of \textit{experience tuples} $\mathcal{D} = \{(s_t, a_t, r_t, s_{t+1}) \}^N_{n=1}$ by rolling out an exploration policy $\pi_e$ for $K$ trajectories. To learn representations that respect Eq.~\ref{eq:det_hom_trans} and~\ref{eq:det_hom_rew}, we minimize the distance between the result of transitioning in observation space, and then mapping to $\aS$, or first mapping to $\aS$ and then transitioning in latent space (see Figure~\ref{fig:action_equivariance}). Additionally, the distance between the observed reward $\R(s)$ and the predicted reward $\aRparam(\projparam(s))$ is minimized. We thus include a general reward loss term. We write $s'_n = \trans(s_n, a_n)$, $z_n = \projparam(s_n)$, and minimize
\begin{align}  
	\mathcal{L}(\theta, \phi, \zeta) = \frac{1}{N} \sum^N_{n=1} \bigg [ &d \left (\projparam(s'_n), \atransparam(z_n, \adeltaparam(z_n, a_n)) \right) \nonumber \\
	+ &d \left (\R(s_n), \aRparam(z_n) \right) \bigg ] \label{eq:basic_loss}
\end{align}
by randomly sampling batches of experience tuples from $\mathcal{D}$. In this paper, we use $d(\as, \as') = \frac{1}{2} (\as - \as')^2$ to model distances in $\aS \subseteq \mathbb{R}^D$. 
Here, $\atransparam$ is a function that maps a point in latent space $\as \in \aS$ to a new state $\as' \in \aS$ by predicting an \textit{action-effect} that acts on $\as$. 
We adopt earlier approaches of letting $\atransparam$ be of the form $\atransparam(\as, \aa)= \as + \adeltaparam(\as, a)$, where $\adeltaparam(\as, a)$ is a simple feedforward network~\cite{cswm, franccois2018combined}. Thus $\adeltaparam: \aS \times \A \rightarrow \aA$ is a function mapping from the original action space to an abstract action space, and $\adeltaparam(\as, a)$ approximates $\bar{A}_s(a)$ (Eq.~$\ref{eq:det_hom_trans}$). The resulting transition loss is a variant of the loss proposed in~\cite{cswm}.
The function $\aRparam: \aS\rightarrow \mathbb{R}$ predicts the reward from $\as$. Since $\proj$, $\atrans$ and $\aR$ are neural networks optimized with SGD, Eq.~\ref{eq:basic_loss} has a trivial solution where all states are mapped to the same point, especially in the sparse reward case. When the reward function is informative, minimizing Eq.~\ref{eq:basic_loss} can suffice, as is empirically demonstrated in~\cite{gelada2019deepmdp}. However, when rewards are sparse, the representations may collapse to the trivial embedding, and for more complex tasks~\cite{gelada2019deepmdp} requires a pixel reconstruction term. In practice, earlier works use a variety of solutions to prevent the trivial map. Approaches based on pixel reconstructions are common~\cite{watter2015embed, watters2019cobra, corneil2018efficient, kurutach2018learning, hafner2018learning, kaiser2019model, ha2018world, zhang2018solar, gelada2019deepmdp}, but there are also approaches based on self-supervision that use alternatives to reconstruction of input states~\cite{anand2019unsupervised, cswm, franccois2018combined, oord2018representation, agrawal2016learning, aytar2018playing, zhang2018composable}. 

To prevent trivial solutions, we use a contrastive loss, maximizing the distance between the latent next state and the embeddings of a set of random other states, $ \mathcal{S}_\neg = \{s_j\}^J_{j=1}$ sampled from the same trajectory on every epoch. Thus, the complete loss is
\begin{align}  
	\mathcal{L}(\theta, \phi, \zeta) = \frac{1}{N} \sum^N_{n=1} \bigg [ &d \left (\projparam(s'_n), \atransparam(z_n, \adeltaparam(z_n, a_n)) \right) \nonumber \\
    + &d \left (\R(s_n), \aRparam(z_n) \right) \nonumber \\
	+ \sum_{s_\neg \in \mathcal{S}_\neg} &d_\neg \left ( \projparam(s_\neg), \atransparam(z_n, \adeltaparam(z_n, a_n)) \right ) \bigg ] \label{eq:actual_loss}
\end{align}
where $d_\neg$ is a negative distance function. Similar to ~\cite{cswm}, we use the hinge loss $d_\neg(\as, \as') = \text{max}(0, \epsilon - d(\as, \as'))$ to prevent the negative distance from growing indefinitely. Here, $\epsilon$ is a parameter that controls the scale of the embeddings.
To limit the scope of this paper, we consider domains where we can find a reasonable data set of transitions without considering exploration. Changing the sampling policy will introduce bias in the data set, influencing the representations. Here we evaluate if we can find plannable MDP homomorphisms and leave the exploration problem to future work. 
\subsection{Constructing the Abstract MDP}
\begin{figure}
    \centering
    \includegraphics[width=0.9\columnwidth]{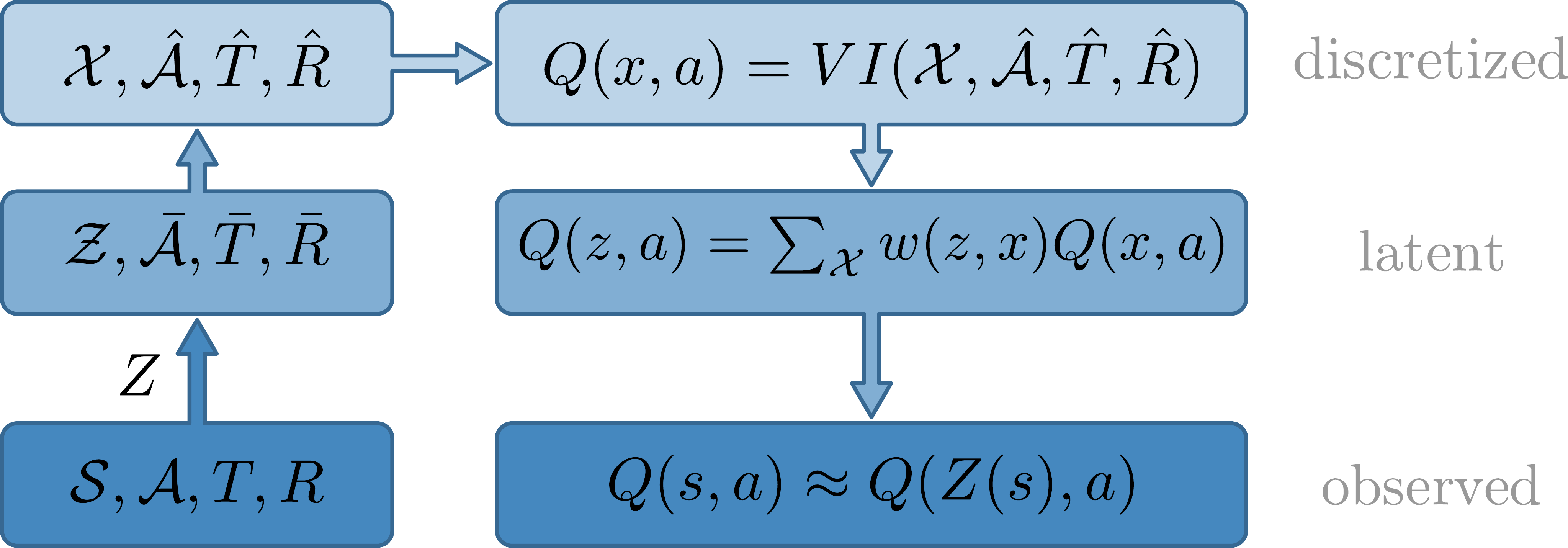}
    \caption{Schematic overview of our method. We learn the map $\proj$ from $\S$ to $\aS$ and discretize $\aS$ to obtain $\daS$. We plan in $\daS$ and use interpolated Q-values to obtain a policy in $\S$.}
    \label{fig:overview}
\end{figure}
After learning a structured latent space, we find abstract MDP $\aM$ by constructing reward and transition functions from $\projparam$, $\atransparam$, $\aRparam$.
\subsubsection{Abstract States}
Core to our approach is the idea that exploiting action-equivariance constraints leads to nicely structured abstract spaces that can be planned in. Of course the space $\aS$ is still continuous, which requires either more complex planning methods, or state discretization. In this paper we aim for the latter, simpler, option, by
constructing a discrete set $\daS$ of (`prototype') latent states in $\aS$ over which we can perform standard dynamic programming techniques. We will denote such prototype states as $\das \in \daS$, cf. Figure~\ref{fig:overview}.
Of course, we then also need to construct discrete transition $\datrans$ and reward $\daR$ functions. The next sub-sections will outline methods to obtain these from $\aS$, $\atransparam$ and $\aRparam$.
To find a `plannable' set of states, the abstract state space should be sufficiently covered. To construct the set, we sample $L$ states from the replay memory and encode them, i.e. $\daS = \{\projparam(\s_l)|\s_l \sim \mathcal{D}\}^L_{l=1}$, pruning duplicates.

\subsubsection{Reward Function}
In Eq.~\ref{eq:actual_loss} we use a reward prediction loss to encourage the latent states to contain information about the rewards. This helps separate distinct states with comparable transition functions. During planning, we can use this predicted reward $\aRparam$. When the reward depends on a changing goal state, such as in the goal-conditioned tasks in Section~\ref{section:experiments}, $\daR (\projparam(s)) = 1$ if $\projparam(s) = \projparam(s_g)$ and $0$ otherwise. We use this reward function in planning, i.e. $\daR(\das) = 1$ if $\das = \projparam(\s_g)$ and $0$ otherwise.
\subsubsection{Transition Function}
	We model the transitions on the basis of similarity in the abstract space. We follow earlier work~\cite{garcia2017few, kipf2018neural} and assume that if two states are connected by an action in the state space, they should be close after applying the latent action transition. The transition function is a distribution over next latent states. Therefore, we use a temperature softmax to model transition probabilities between representations of abstract states in $\daS$:
 \begin{align}
	 \predatrans(z_j|z_i, \alpha) &= \frac{e ^{-d(\as_j, \as_i + \adeltaparam(\as_i, \alpha) )/\tautrans } }{\sum_{k \in \daS} e^{-d(\as_k, \as_i + \adeltaparam(\as_i, \alpha))/\tautrans }}
 \end{align}
	Thus, for the transitions between abstract states:
\begin{align}
	\datrans(x=j|x'=i, \daa=\alpha) &= \predatrans(z_j|z_i, \alpha)
 \end{align}
 where $\tautrans$ is a temperature parameter that determines how `soft' the edges are, and $z_j$ is the representation of abstract state $j$.
Intuitively, this means that if an action moves two states closer together, the weight of their connection increases, and if it moves two states away from each other, the weight of their connection decreases. For very small $\tautrans$, the transitions are deterministic.

\subsubsection{Convergence to an MDP homomorphism}
We now show that when combining optimization of our proposed loss fuction \eqref{eq:actual_loss} with the construction of an abstract MDP as detailed in this subsection, we can approximate an MDP homomorphism. Specifically, for deterministic MDPs, we show that when the loss function in Eq.~\ref{eq:actual_loss} reaches zero, we have an MDP homomorphism of $\M$.
\begin{theorem}
	In a deterministic MDP $\M$, assuming a training set that contains all state, action pairs, and an exhaustively sampled set of abstract states $\daS$ we consider a sequence of losses in a successful training run, i.e. the losses converge to 0. In the limit of the loss $\mathcal{L}$ in Eq.~\ref{eq:actual_loss} approaching 0, i.e. $\mathcal{L} \rightarrow 0$ 
        and 
        $0 < \tautrans \ll 1$, $\tautrans \ll \epsilon$, 
        $\hom = (\projparam, \adeltaparam)$ is an MDP homomorphism of $\M$. 
\end{theorem}
\begin{proof}
\sloppy    
	Fix $0 < \tautrans \ll 1$ and write $\as = \projparam(\s)$ and $\aa = \adeltaparam(\as, \a)$. Consider that learning converges, i.e. $\mathcal{L} \rightarrow 0$. This implies that the individual loss terms $d(\atransparam(\as, \aa), \as')$, $d_\neg(\atransparam(\as, \aa), \as_\neg)$ and $d(\R(\s), \aRparam(\as))$ also go to zero for all $(\s, \a, r, \s', s_\neg) \sim \mathcal{D}$. \\
        \textbf{Positive samples:}        
        As the distance for positive samples $d_+ = d(\atransparam(\as, \aa), \as') \rightarrow 0$, then $d_+ \ll \tau$. Since $d_+ \ll \tau$, then $e^{-d_+/\tau} \approx 1$.\\
        \textbf{Negative samples:}        
            Because the negative distance $d_\neg(\atransparam(\as, \aa), \as_\neg) \rightarrow 0$, 
            $d_\neg \leq \epsilon$. 
            This, in turn, implies that the distance 
            to all negative samples $d_- = d(\atransparam(\as, \aa), \as_\neg) \geq \epsilon  $ 
            and thus $\tau \ll \epsilon \leq d_-$, 
            meaning that $1 \ll \frac{d_-}{\tau}$ and thus $e^{-d_-/\tau} \approx 0$.\\
	This means that when the loss approaches 0, $\predatrans(\as'|\as, \aa)=1$ where $\trans(\s'|\s, \a)=1$ and $\predatrans(\as_\neg|\as, \aa)=0$ when $\trans(\s_\neg|\s, \a)=0$. 
Since $\M$ is deterministic, $\trans(\s'|\s, \a)$ transitions to one state with probability 1, and probability 0 for the others. 
Therefore, 
$
	\predatrans(\projparam(\s')| \projparam(\s), \adeltaparam(\projparam(\s), \a)) = \sum_{s'' \in [s']_{\proj}} \trans(\s''|\s, \a) 
$
	and Eq.~\ref{eq:det_hom_trans} holds.
        As the distance for rewards
        $d(\R(\s), \aRparam(\as)) \rightarrow 0$, we have that
	$ \aRparam(\as) = \R(\s) $
	and Eq.~\ref{eq:det_hom_rew} holds.
	Therefore, when the loss reaches zero we have an MDP homomorphism of $\M$.
\end{proof}
Note that Eq.~\ref{eq:actual_loss} will not completely reach zero: negative samples are drawn uniformly. Thus,  a positive sample may occasionally be treated as a negative sample. Refining the negative sampling can further improve this approach.

\subsection{Planning and Acting}
After constructing the abstract MDP we plan with VI~\cite{puterman1994markov} and lift the found policy to the original space by interpolating between Q-value embeddings. 
Given $\daM = (\daS, \daA, \datrans, \daR, \gamma)$, VI finds a policy $\dapi$ that is optimal in $\daM$. For a new state $\s^* \in \S$, we embed it in the representation space $\aS$ as $\as^* = \projparam(\s^*)$ and use a softmax over its distance to each $\das \in \daS$ to interpolate between their Q-values, i.e.
\begin{align}
    Q(\as^*, a) &= \sum_{\das \in \daS} w(\as^*, \das) Q(\das, a) \\
	w(\as^*, x) &= \frac{e^{-d(\as_x, z^*) /\eta } }{\sum_{k \in \daS} e^{-d(\as_k, \as^*)/\eta }} \label{interpol}
\end{align}
where $\eta$ is a temperature parameter that sets the `softness' of the interpolation. We use the interpolated Q-values for greedy action selection for $s^*$, transition to $s^{**}$ and iterate until the episode ends.

%% file: experiments.tex
Here we show that in simple domains, our approach 1) succeeds at finding plannable MDP homomorphisms for discrete and continuous problems 2) requires less data than model-free approaches, 3) generalizes to new reward functions and data and 4) trains faster than approaches based on reconstructions. We focus on deterministic MDPs. While preliminary results on stochastic domains were promising, an in-depth discussion is beyond the scope of this paper.

\begin{figure*}
	\begin{subfigure}{0.24\textwidth}
	    \centering
	    \includegraphics[width=\textwidth]{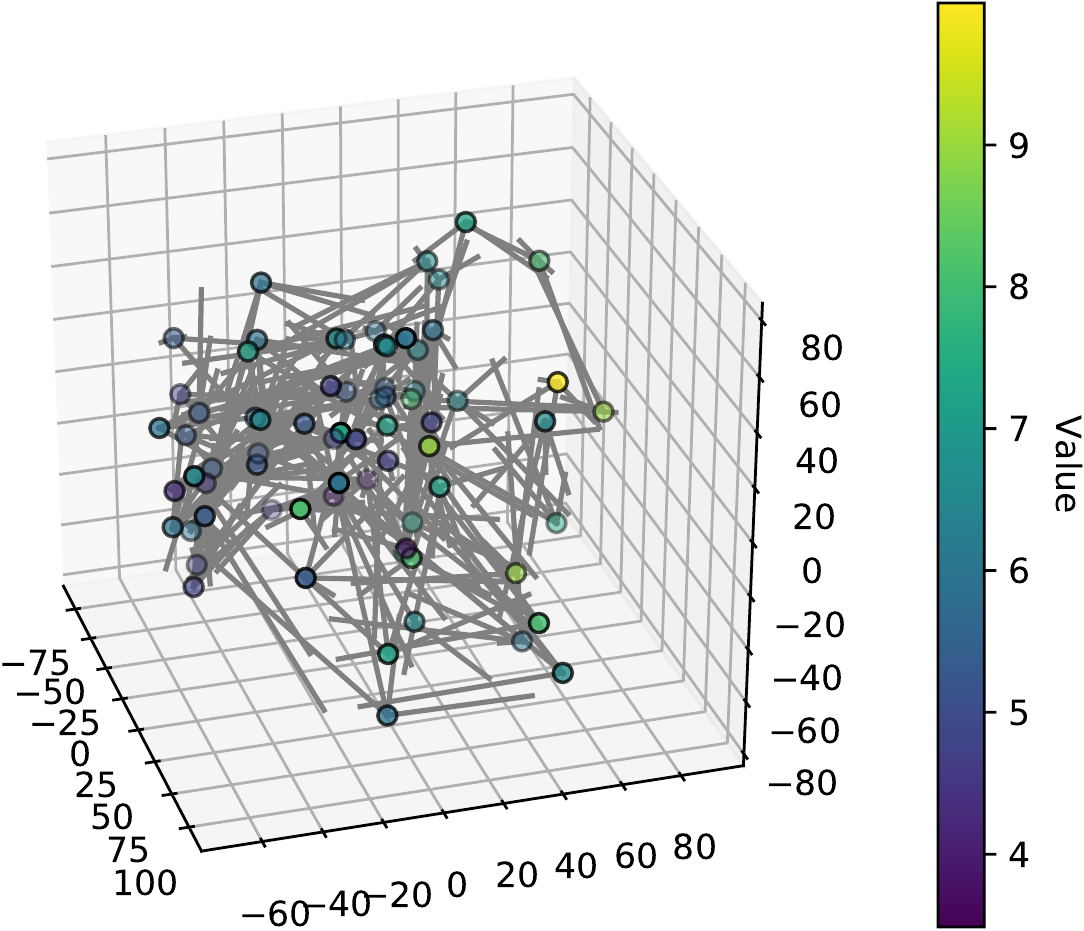}
	    \caption{WM-AE Baseline}
	    \label{fig:room_wm}
	\end{subfigure}
	\begin{subfigure}{0.24\textwidth}
	    \centering
	    \includegraphics[width=\textwidth]{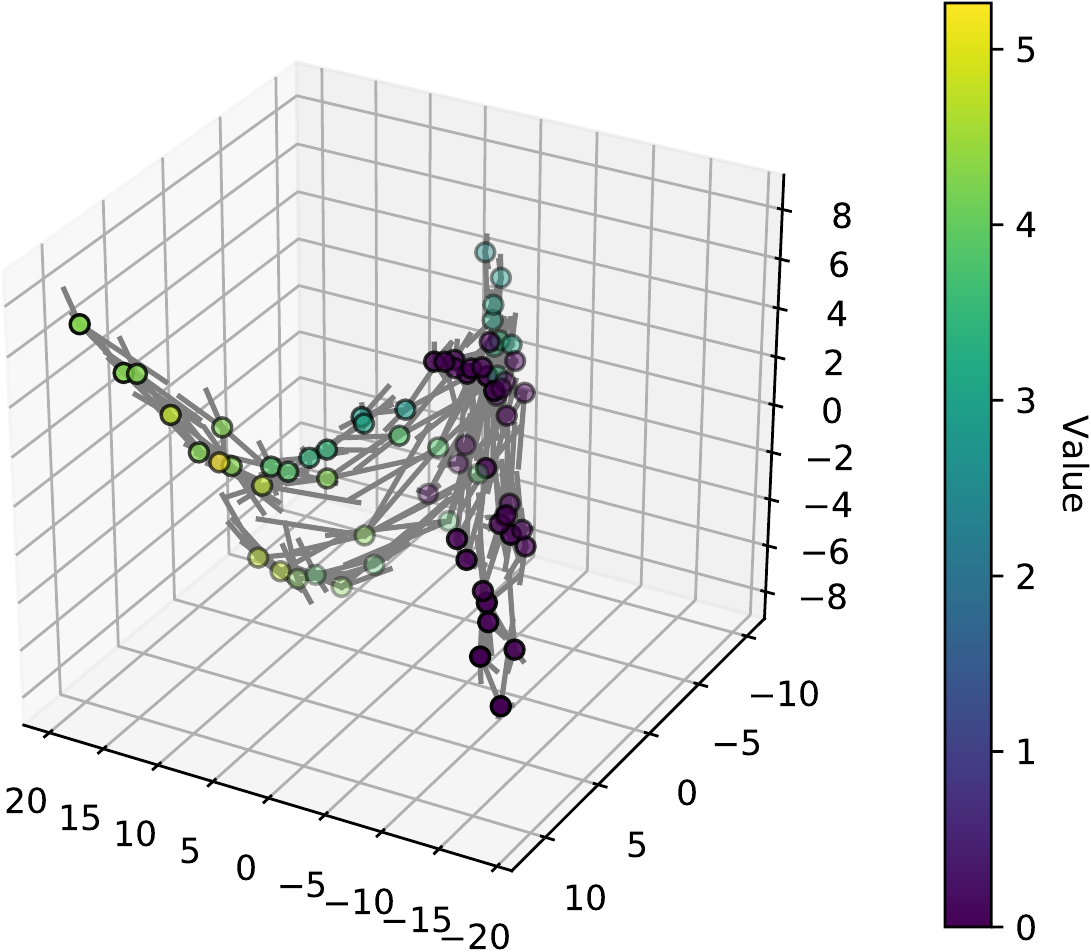}
	    \caption{LD-AE Baseline}
	    \label{fig:room_recon}
	\end{subfigure}
	\begin{subfigure}{0.24\textwidth}
	    \centering
	    \includegraphics[width=\textwidth]{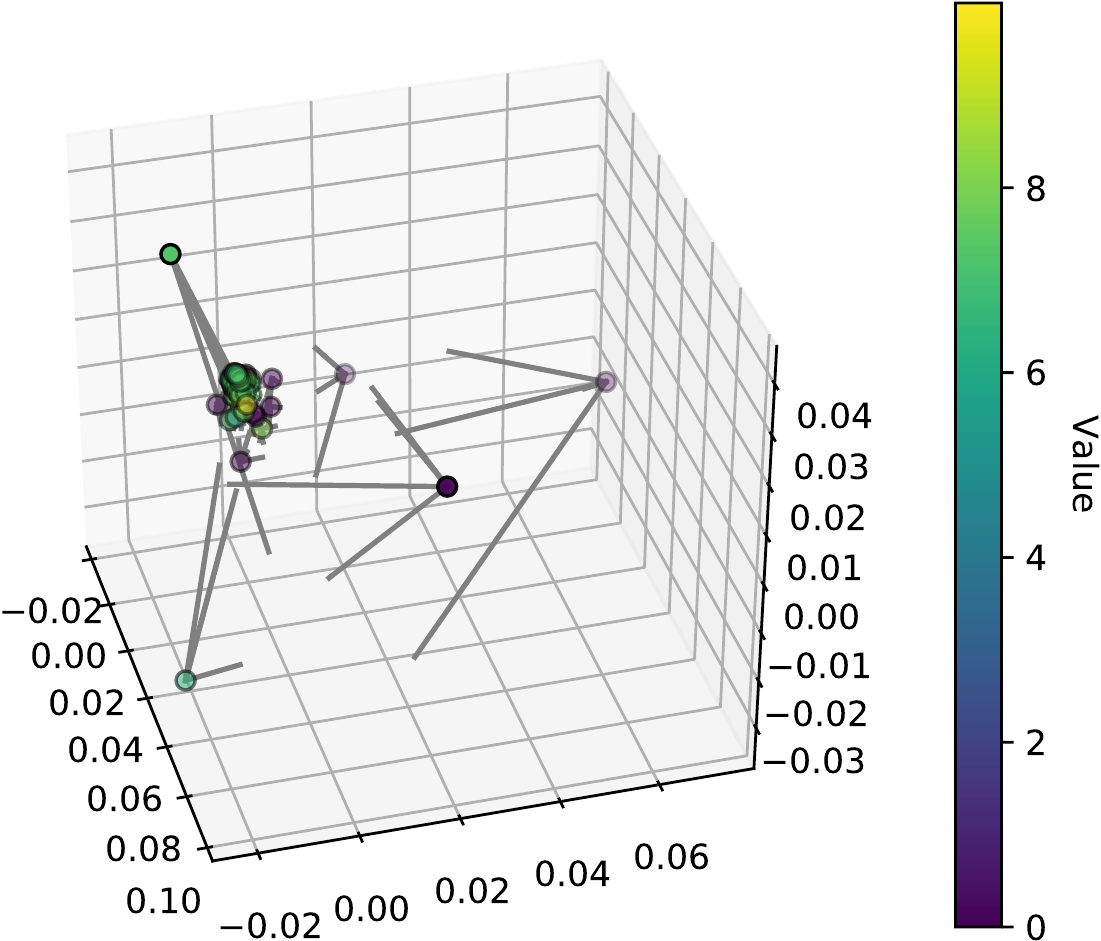}
	    \caption{DMDP-H Baseline}
	    \label{fig:room_aceq0}
	\end{subfigure}
	\begin{subfigure}{0.24\textwidth}
	    \centering
	    \includegraphics[width=\textwidth]{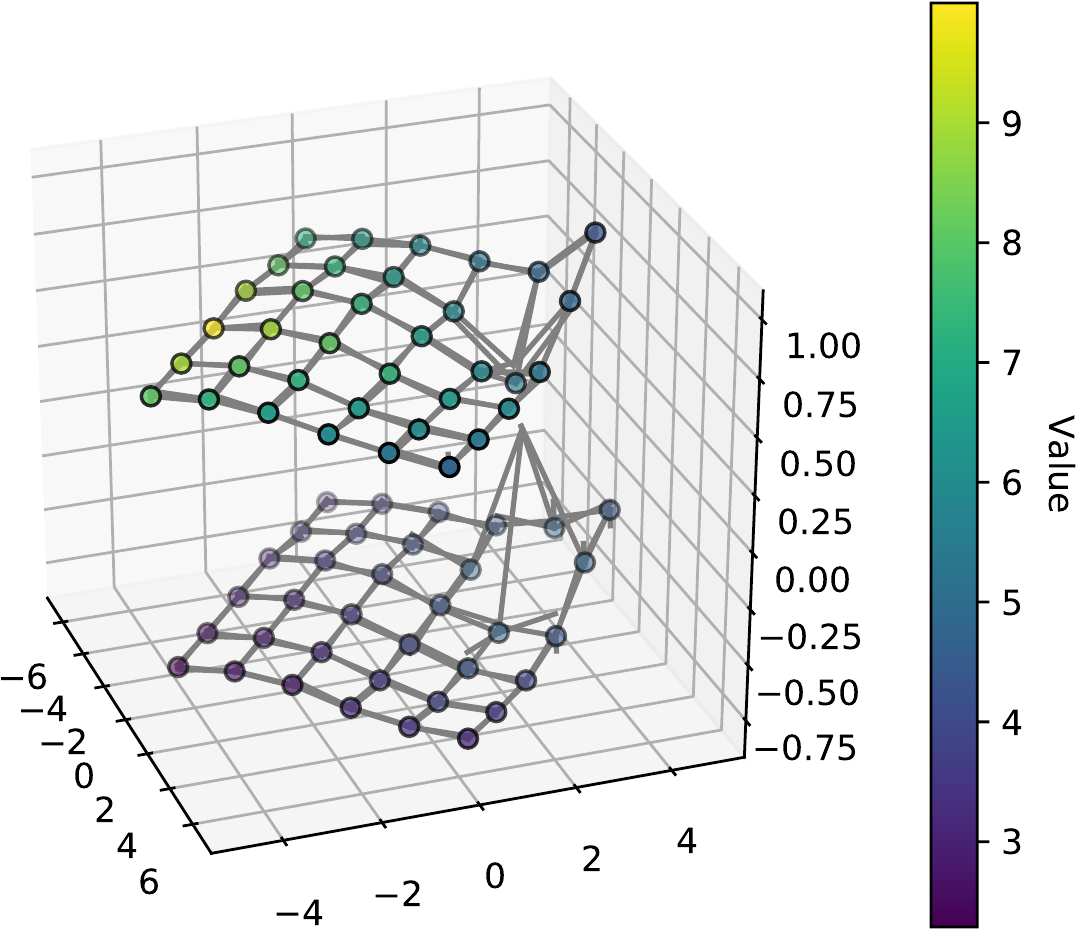}
	    \caption{This paper}
	    \label{fig:room_aceq5}
	\end{subfigure}
	\caption{Abstract MDP for three approaches in the single object room domain. Nodes are PCA projections of abstract states, edges are predicted $\atransparam$, colors are predicted values. }
	\label{fig:room_latents}
\end{figure*}
\subsection{Baselines}
To evaluate our approach, we compare to a number of baselines:
\begin{enumerate}
    \item WM-AE: An auto-encoder approach inspired by World Models~\cite{ha2018world}. We follow their approach of training representations using a reconstruction loss, then learning latent dynamics on fixed representations. We experimented with a VAE~\cite{kingma2013auto}, which did not perform well (see~\cite{cswm} for similar results). We thus use an auto-encoder to learn an embedding, then train an MLP to predict the next state from embedding and action.
    \item LD-AE: An auto-encoder with latent dynamics. We train an auto-encoder to reconstruct the input, and predict the next latent state. We experimented with reconstructing the next state, but this resulted in the model placing the next state embeddings in a different location than the latent transitions.
    \item DMDP-H: We evaluate the effectiveness of training without negative sampling. This is similar to DeepMDP~\cite{gelada2019deepmdp}. However, unlike DeepMDP, DMDP-H uses action-embeddings, for a fairer comparison.
    \item GC-Model-Free: Finally, we compare to a goal-conditioned model-free baseline (REINFORCE with state-value baseline), to contrast our approach with directly optimizing the policy\footnote{Deep reinforcement learning algorithms such as our baseline may fail catastrophically depending on the random seed~\cite{henderson2017deep}. For a fair comparison, we train the baseline on $6$ random seeds, then remove those seeds where the method fails to converge for the train setting.}.  We include the goal state as input for a fair comparison.
\end{enumerate}
\begin{wrapfigure}{r}{0.5\columnwidth}
	\vspace{-3em}
	\includegraphics[width=0.5\columnwidth]{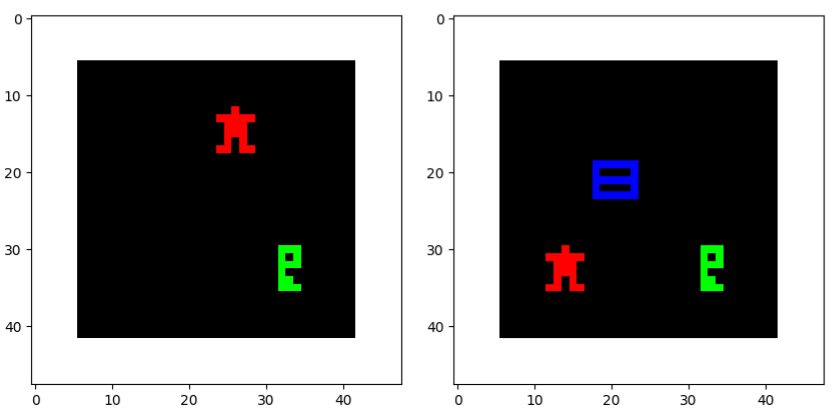}
\caption{Example states in the object collection domain for the single object and double object tasks.}
\label{fig:objects}
\end{wrapfigure}
To fairly compare the planning approaches, we perform a grid search over the softness of the transition function by evaluating performance on the train goals in $\tautrans \in [1, 0.1, 0.001, 0.0001, 0.00001,$
$\mathrm{1e-20}]$. 
Unless otherwise stated, the planning approaches are all trained on datasets of 1000 trajectories, sampled with a random policy. The learning rate is set to 0.001 and we use Adam~\cite{kingma2014adam}. For the hinge loss, we use $\epsilon=1$. The latent dimensionality is set to 50 everywhere. Our approach is trained for 100 epochs. WM-AE is trained for 1000 epochs in total: 500 for the auto-encoder and 500 for the dynamics. LD-AE is trained for 1000 epochs. For constructing the abstract MDP we sample $1024$ states from $\mathcal{D}$, project unto $\aS$ and prune duplicates. For planning we use VI with discount factor $\gamma=0.9$, 500 backups and interpolation parameter (Eq.~\ref{interpol}) $\eta=\mathrm{1e-20}$. The learning rate for the model-free baseline was chosen by fine-tuning on the training goals. For the model-free baseline, we use a learning rate of $\mathrm{5e-4}$ and we train for 500k steps (more than five times the number of samples the planning approaches use). Network $Z_\theta$ has 2 convolutional layers (both 16 channels, $3\times3$ filters) and 3 fully connected layers (input$\rightarrow 64 \rightarrow 32 \rightarrow |z|$). Networks $T_\phi$ and $R_\xi$ each have 2 fully connected layers. We use ReLU non-linearities between layers. 

\subsection{Object Collection}
We test our approach on an object collection task inspired by the key task in~\citep{franccois2018combined}, with major differences: rather than searching for three keys in a labyrinth, the agent is placed in a room with some objects. Its task is to collect the key. On every time step, the agent receives a state---a $3 \times 48 \times 48$ pixel image (a channel per object, including the agent), as shown in Figure~\ref{fig:objects}---and a goal state of the same size. 
At train time, the agent receives reward of $1$ on collection of the \textit{key} object, and a reward of $-1$ if it grabs the wrong object, and a reward of $-0.1$ on every time step. The episode ends if the agent picks up one (or more) of the objects and delivers it to one of the four corners (randomly sampled at episode start), receiving an additional delivery reward of $1$. At test time, the agent is tasked with retrieving one of the objects chosen at random, and delivering to a randomly chosen location, encoded as a desired goal state. This task will evaluate how easily the trained agent adapts to new goals/reward functions. The agent can interact with the environment until it reaches the goal or $100$ steps have passed.  For both tasks, we compare to the model-free baseline. We also compare to the DMDP-H, WD-AE and LD-AE baselines.
We additionally perform a grid search over the hinge, number of state samples for discretization and $\eta$ hyperparameters for insight in how these influence the performance. This showed that our approach is robust with respect to the hinge parameter, but it influences the scale of the embeddings. The results decrease only when using 256 or fewer state samples. Lastly, $\eta$ is robust for values lower than 1. We opt for a low value of $\eta$, to assign most weight to the Q-value of the closest state.

\subsubsection{Single Object Task}
We first evaluate a simple task with only one object (a key). The agent's task is to retrieve the key, and move to one of four delivery locations in the corners of the room. The delivery location is knowledge supplied to the agent in the form of a goal state that places the agent in the correct corner and shows that there is no key. These goal states are also supplied to the baseline, during training and testing. Additionally, we perform an ablation study on the effect of the reward loss.
\begin{table}
\small
\begin{tabular}{lrrrr}
\toprule
	 \multicolumn{1}{c}{Avg. ep. length $\downarrow$} \\
Task & \multicolumn{2}{c}{\textbf{Single Object}} & \multicolumn{2}{c}{\textbf{Double Object}}\\
\emph{Goal Set} & Train & Test & Train & Test \\
\midrule
GC-Model-free & 10.00 \pmin{0.11} & 67.25 \pmin{6.81} & 10.10 \pmin{0.69} & 38.25 \pmin{15.30}  \\
	WM-AE & 12.96  \pmin{8.93} & 10.03 \pmin{5.56}  &  29.61 \pmin{19.42} & 22.53 \pmin{22.12} \\
	LD-AE & 23.46 \pmin{27.10} & 21.04 \pmin{21.71} & 60.26 \pmin{29.14} & 52.72 \pmin{27.32}  \\
	DMDP-H ($J=0$) & 82.88 \pmin{11.62} & 85.69 \pmin{7.98} & 81.24 \pmin{2.45} & 81.17 \pmin{2.69} \\
\rowcolor{Gray}
	Ours, $J=1$, & 8.61 \pmin{0.35} & \textbf{7.53} \pmin{0.24} & 8.53 \pmin{0.36} & \textbf{8.38} \pmin{0.07} \\
\rowcolor{Gray}
	Ours, $J=3$ & 8.68 \pmin{0.27} & 7.63 \pmin{0.19} & 8.61 \pmin{0.38} & 8.95 \pmin{0.63}\\
\rowcolor{Gray}
	Ours, $J=5$ & \textbf{8.57} \pmin{0.48} & 7.74 \pmin{0.22} & \textbf{8.26} \pmin{0.84} & 8.96 \pmin{1.15} \\
\bottomrule
\end{tabular}
	\caption{Comparing average episode length of 100 episodes on the object collection domain. Reporting mean and standard deviation over 5 random seeds for the planning approaches. The model free approach is averaged over 4 random seeds for the single object domain, 3 random seeds for the double object domain.}
\label{tab:grid_world}
\end{table}
The average episode lengths are shown in Table~\ref{tab:grid_world}. Our approach outperforms all baselines, both at train and at test time. There is no clear preference in terms of the number of negative samples --- as long as $J>0$ --- the result for all values of $J$ are quite close together. The DMDP-H approach fails to find a reasonable policy, possibly due to the sparse rewards in this task providing little pull against state collapse. Out of the planning baselines, WM-AE performs best, probably because visually salient features are aligned with decision making features in this task. Finally, the model-free approach is the best performing baseline on the training goals, but does not generalize to test goals. \\
The results of the reward ablation are shown in Table~\ref{tab:reward}. While removing the reward loss does not influence performance much for $J=0$, $J=3$ and $J=5$, when $J=1$ the reward prediction is needed to separate the states. Without the reward, the single negative sample does not provide enough pull for complete separation.
\begin{table} 
\small
\begin{tabular}{lrrrr}
\toprule
	 \multicolumn{1}{c}{Avg. ep. length $\downarrow$} \\
 & \multicolumn{2}{c}{\textbf{Reward Loss}} & \multicolumn{2}{c}{\textbf{No Reward Loss}}\\
\emph{Goal Set} & Train & Test & Train & Test \\
\midrule
	DMDP-H ($J=0$) &  82.88 \pmin{11.62} & 85.69 \pmin{7.98}&  87.03 \pmin{3.08} & 84.08 \pmin{3.02}\\
	Ours, $J=1$ & 8.61 \pmin{0.35} & 7.53 \pmin{0.24} & 74.32 \pmin{19.90} & 68.54 \pmin{17.29}\\
	Ours, $J=3$ & 8.68 \pmin{0.27} & 7.63 \pmin{0.19}& 8.54 \pmin{0.36} & 7.44 \pmin{0.21} \\
	Ours, $J=5$ & 8.57 \pmin{0.48} & 7.74 \pmin{0.22}  & 8.52 \pmin{0.19} & 7.53 \pmin{0.20}   \\
\bottomrule
\end{tabular}
	\caption{Ablation study of the effect of the reward loss. Comparing average episode length of 100 episodes for the single object room domain. Reporting mean and standard deviation over 5 random seeds.}
\label{tab:reward}
\end{table}

\begin{figure}
	\begin{subfigure}[t]{0.8\columnwidth}
	    \centering
	    \includegraphics[width=\columnwidth]{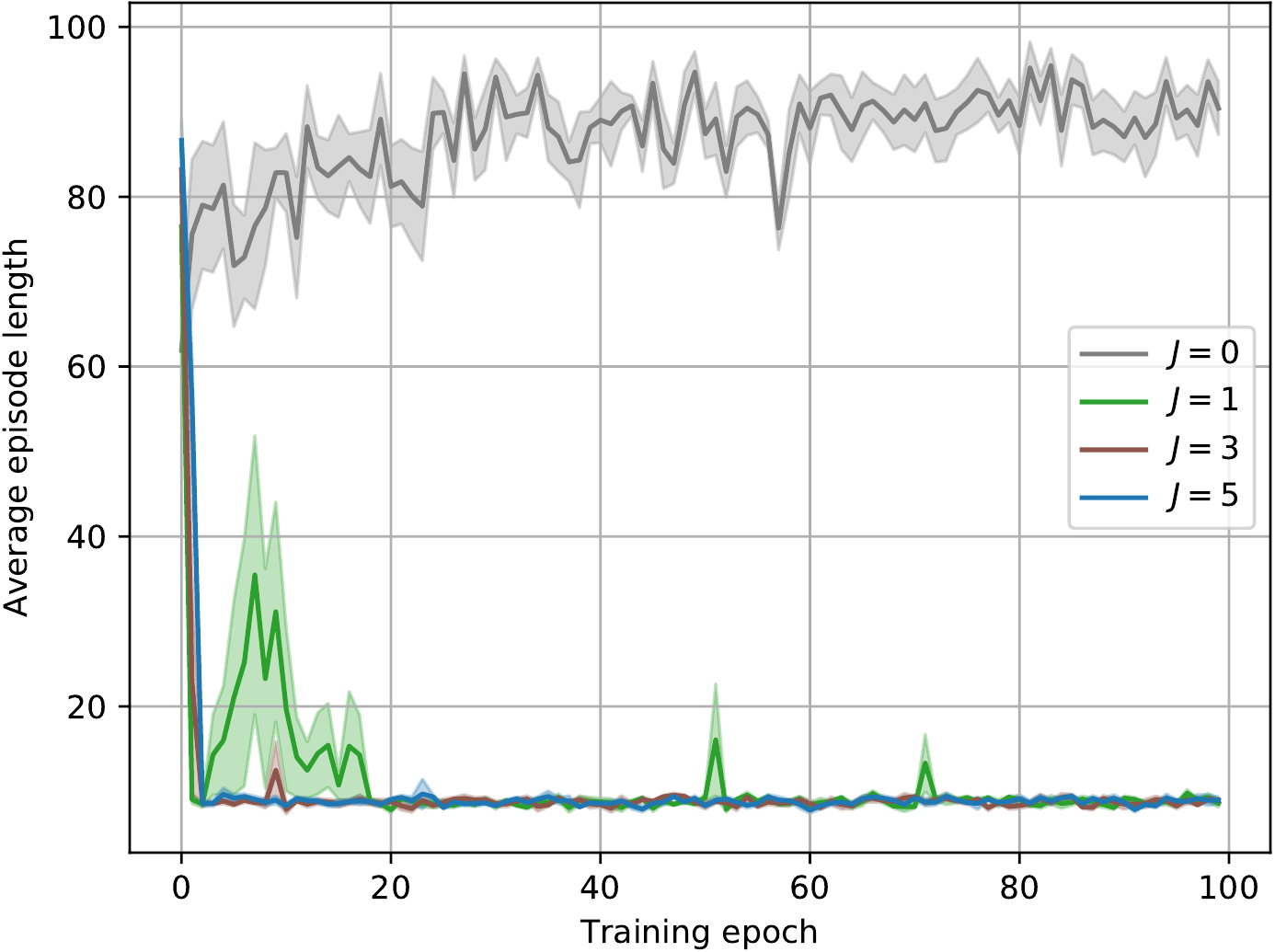}
		\caption{Comparison of different values of $J$.}
	    \label{fig:epochs_aceq}
	\end{subfigure}
	\begin{subfigure}[t]{0.8\columnwidth}
	    \centering
	    \includegraphics[width=\columnwidth]{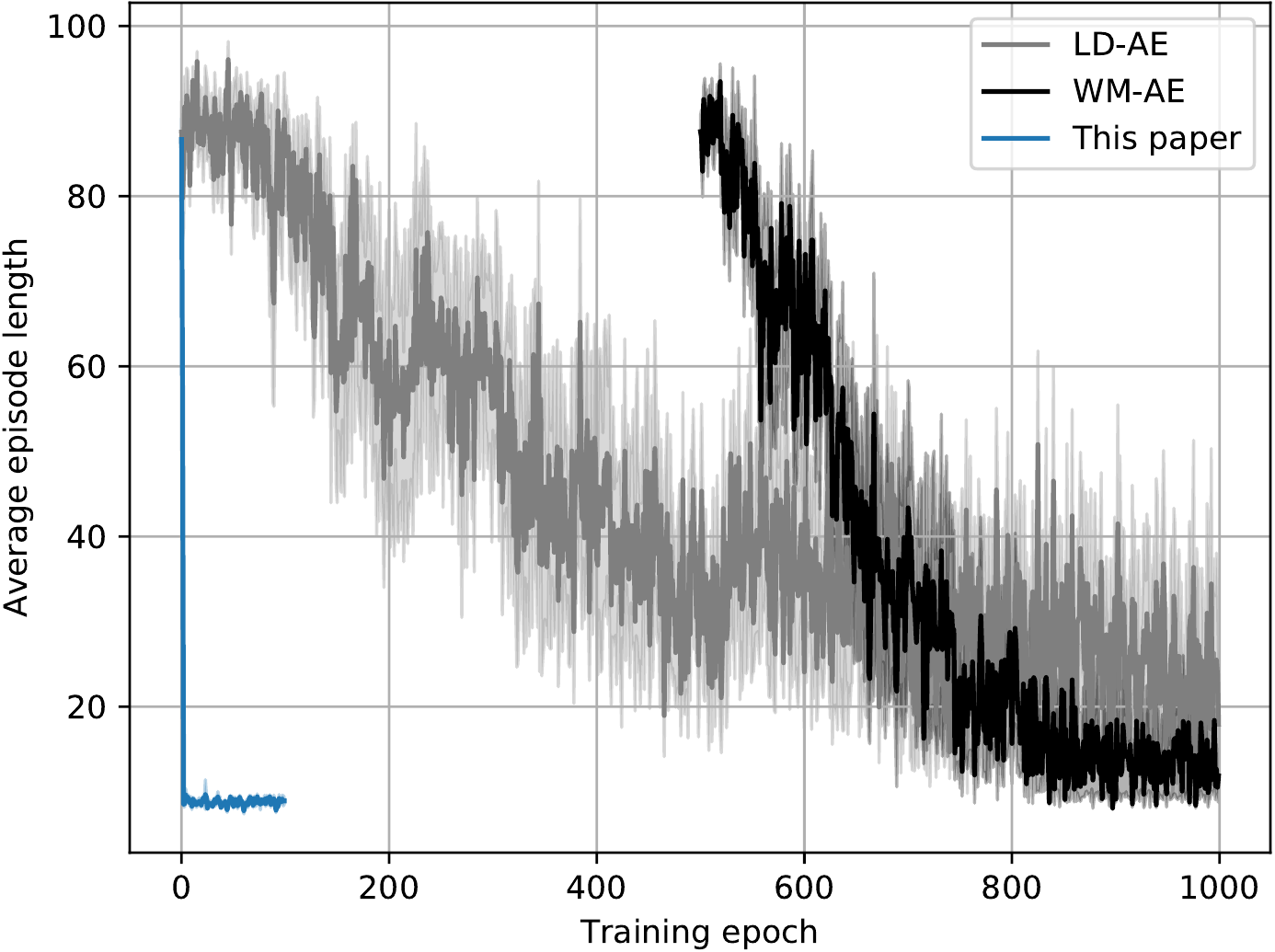}
	    \caption{Comparison of this paper and the WM-AE and LD-AE baselines. WM-AE can not be evaluated until the auto-encoder has finished training and training of the dynamics model begins.}
	    \label{fig:epochs_baseline}
	\end{subfigure}
	\caption{Average episode length per training epoch for the single object domain. Reported mean and standard error over 5 random seeds.  }
	\label{fig:epochs_plot}
\end{figure}
\begin{figure*}

	\begin{subfigure}{0.24\textwidth}
	    \centering
	    \includegraphics[width=\textwidth]{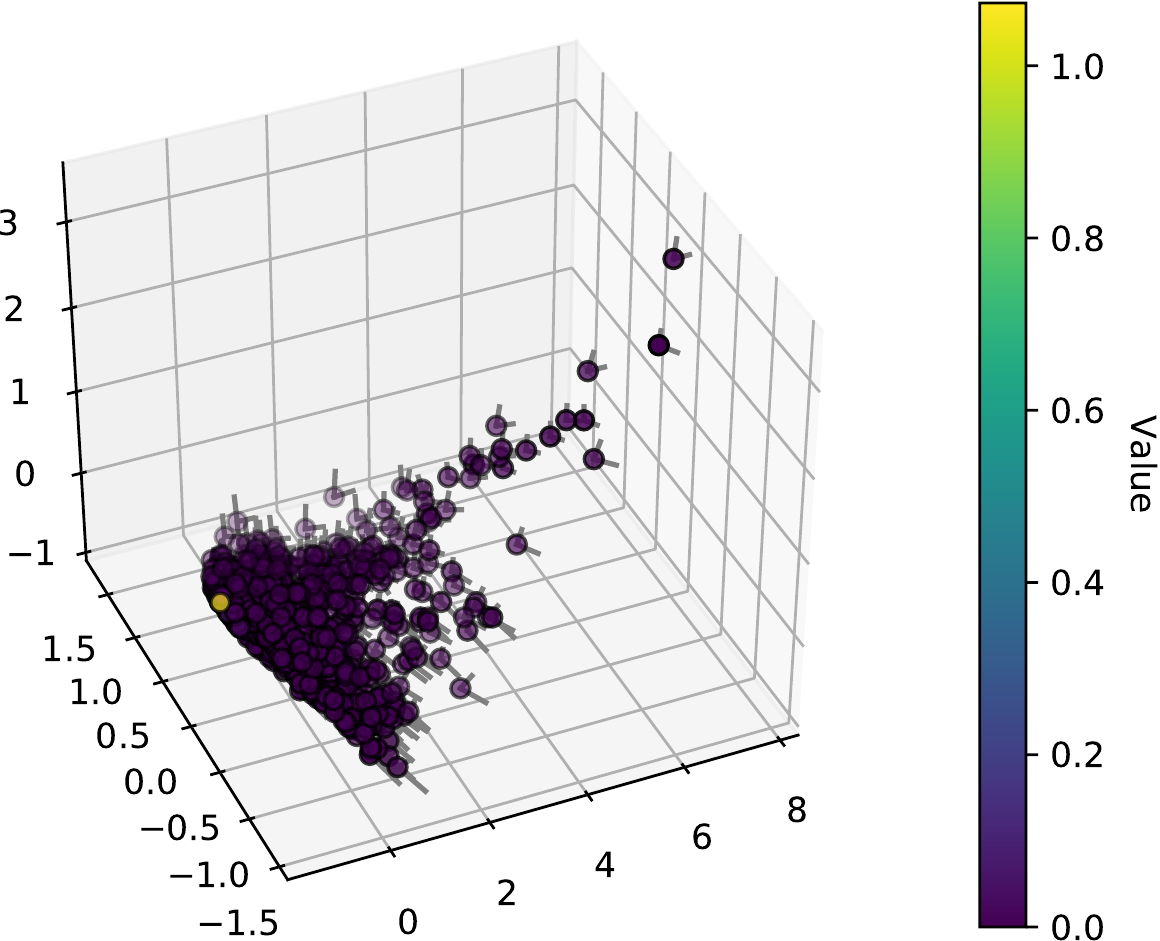}
	    \caption{WM-AE Baseline}
	    \label{fig:cartpole_wm}
	\end{subfigure}
	\begin{subfigure}{0.24\textwidth}
	    \centering
	    \includegraphics[width=\textwidth]{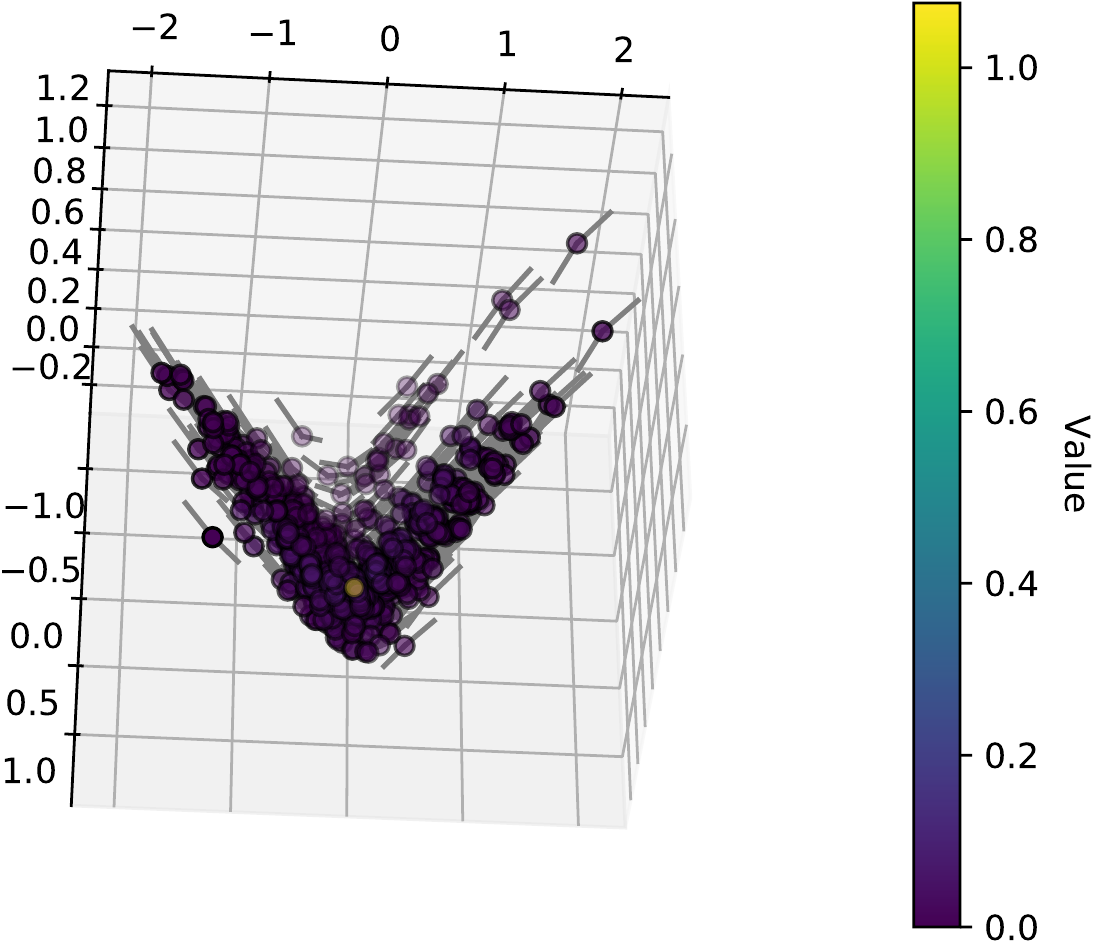}
	    \caption{LD-AE Baseline}
	    \label{fig:cartpole_recon}
	\end{subfigure}
	\begin{subfigure}{0.24\textwidth}
	    \centering
	    \includegraphics[width=\textwidth]{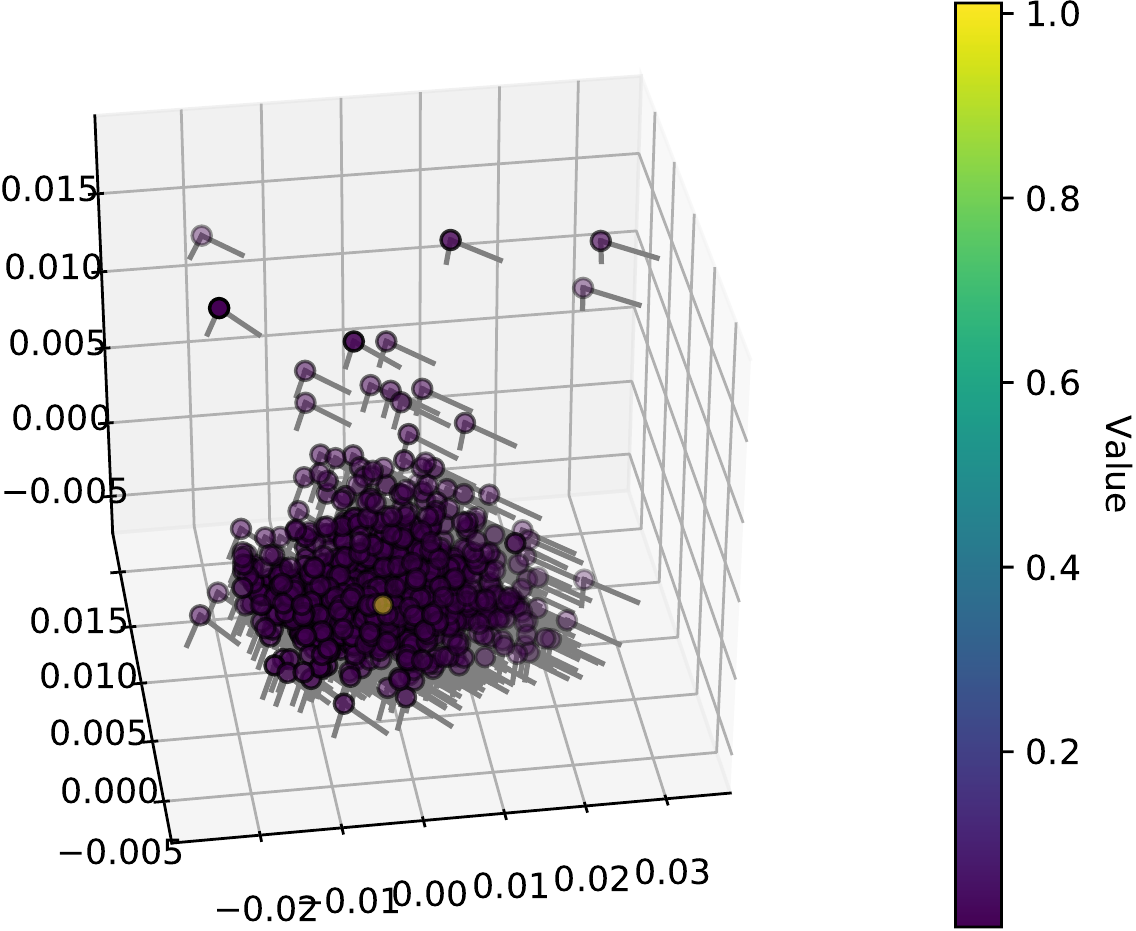}
	    \caption{DMDP-H Baseline}
	    \label{fig:cartpole_aceq0}
	\end{subfigure}
	\begin{subfigure}{0.24\textwidth}
	    \centering
	    \includegraphics[width=\textwidth]{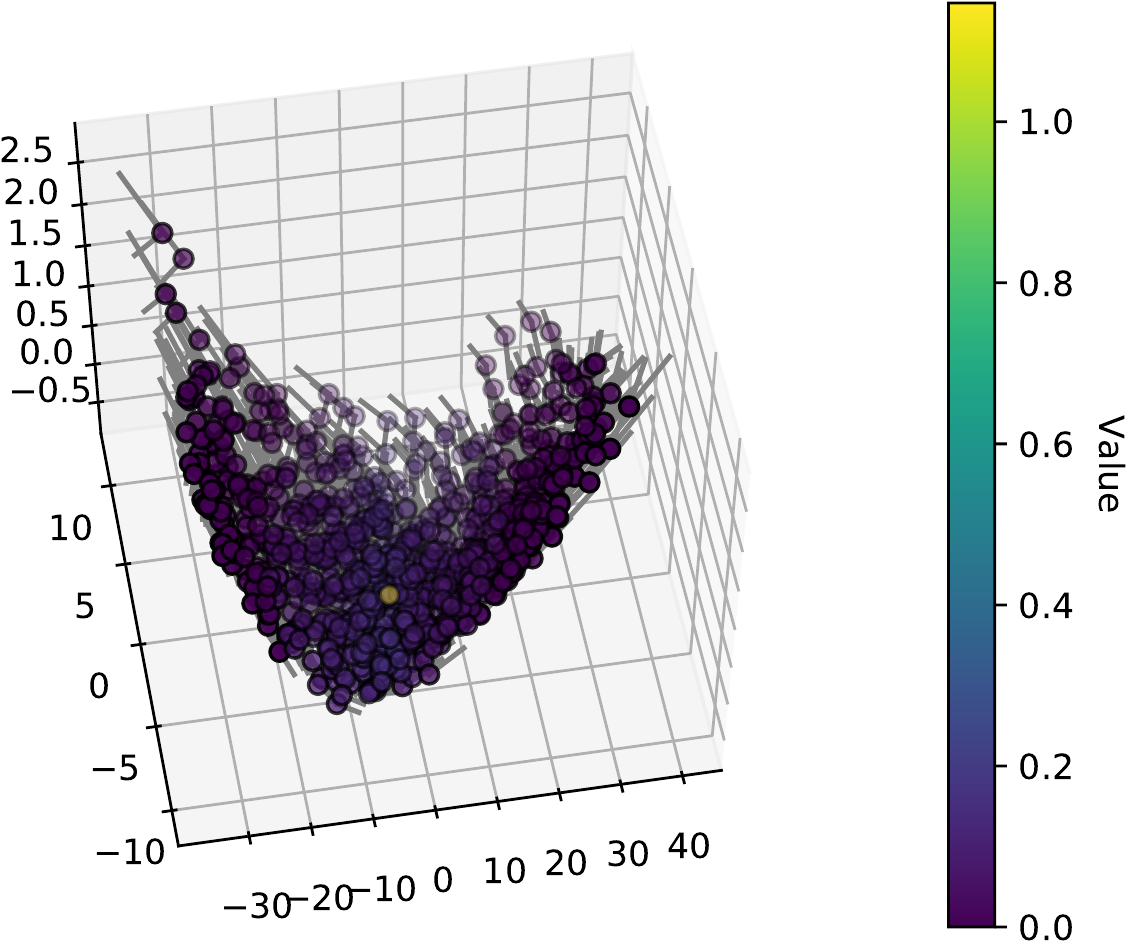}
	    \caption{This paper}
	    \label{fig:cartpole_aceq5}
	\end{subfigure}
	\caption{Abstract MDP for four approaches in CartPole. Nodes are PCA projections of abstract states, edges are predicted $\atransparam$, colors are predicted values. }
	\label{fig:cartpole_latents}
\end{figure*}
We show the latent spaces found for the baselines and our approach in Figure~\ref{fig:room_latents}. Our approach has found a double grid structure - representing the grid world before, and after picking up the key. The baselines are reasonably plannable after training for long enough, but the latent spaces aren't as nicely structured as our approach. This mirrors results in earlier work~\cite{cswm}. Thus, while pixel reconstruction losses may be able to find reasonable representations for certain problems, these rely on arbitrarily complex transition functions. Moreover, due to their need to train a pixel reconstruction loss they take much longer to find useable representations. This is shown in Figure~\ref{fig:epochs_baseline}, where the performance after planning for each training epoch is plotted and compared. Additionally, we observe state collapse for DMDP-H in Figure~\ref{fig:room_aceq0}, and this is reflected in a high average episode length after planning.

\subsubsection{Double Object Task}
We now extend the task to two objects: a key and an envelope. The agent's task at train time is still to retrieve the key. At test time, the agent has to pick up the key or the envelope (randomly chosen) and deliver it to one of the corners. 
We show results in Table~\ref{tab:grid_world}. Again, our method performs well on both train and test set, having clearly learned a useful abstract representation, that generalizes to new goals. The WM-AE baseline again fares better than the LD-AE baseline, and DMDP-H fails to find a plannable representation. The model-free baseline performs slightly worse than our method on this task, even after seeing much more data. Additionally, even though it performs reasonably well on the training goals, it does not generalize to new goals at all. The WM-AE performs worse on this task than our approach, but generalizes much better than the model-free baseline, due to its planning, while the LD-AE baseline does not find plannable representations of this task.  

\subsection{Continuous State Spaces}
We evaluate whether we can use our method to learn plannable representations for continuous state spaces. We use OpenAI's CartPole-v0 environment~\cite{brockman2016openai}. We include again a model-free baseline that is trained until completion as a reference for the performance of a good policy. We also compare DMDP-H, WD-AE and LD-AE. We expect that the latter two would perform well here; after all, the representation that they reconstruct is already quite compact. We additionally evaluate performance when the amount of data is limited to only 100 trajectories (and we limit the number of training epochs for all planning approaches to 100 epochs). 
We plot the found latent space for our approach and the baselines in Figure~\ref{fig:cartpole_latents}. The goal in this problem is to reach the all-zero reward vector, which we set as the goal state with reward 1, and all other states to reward 0. For our approach and both auto-encoder baselines, the latent space forms a bowl with the goal in its center. The DMDP-H again shows a shrunk latent space, and does not have this bowl structure.
Results are shown in Table~\ref{tab:cartpole}. Our approach performs best out of all planning approaches. When trained fully, the model-free approach performs better. However, when we limit the number of environmental interactions to 100 trajectories, we see that the planning approach still finds a reasonable policy, while the model-free approach fails completely. This indicates that our approach is more data efficient.
\begin{table}
\small
	\centering
\begin{tabular}{lrr}
\toprule
	 Average episode length $\uparrow$& \textbf{Standard}  & \textbf{Only 100 trajectories}  \\
\midrule
	GC-Model-free & \textbf{197.85} \pmin{2.16} & 23.84 \pmin{0.88}\\
	WM-AE & 150.61 \pmin{30.48} & 114.47 \pmin{17.32} \\
	LD-AE & 157.10 \pmin{11.14} & 154.73 \pmin{50.49} \\
	DMDP-H ($J=0$) & 39.32 \pmin{9.02} & 72.81 \pmin{20.16} \\
\rowcolor{Gray}
	Ours, $J=1$, & 174.64 \pmin{22.43} & 127.37 \pmin{44.02} \\
\rowcolor{Gray}
	Ours, $J=3$ & 166.05 \pmin{24.73} & 148.30 \pmin{67.27}\\
\rowcolor{Gray}
	Ours, $J=5$ & 186.31 \pmin{12.28} & \textbf{171.53} \pmin{34.18} \\
\bottomrule
\end{tabular}
\caption{CartPole results. Comparing average episode length over 100 episodes, reporting mean and standard deviation over 5 random seeds. The left column has standard settings, in the right column only 100 trajectories are encountered, and planning models are trained for only 100 epochs.}
\label{tab:cartpole}
\end{table}
\begin{figure*}
	\begin{subfigure}{0.24\textwidth}
	    \centering
	    \includegraphics[width=\textwidth]{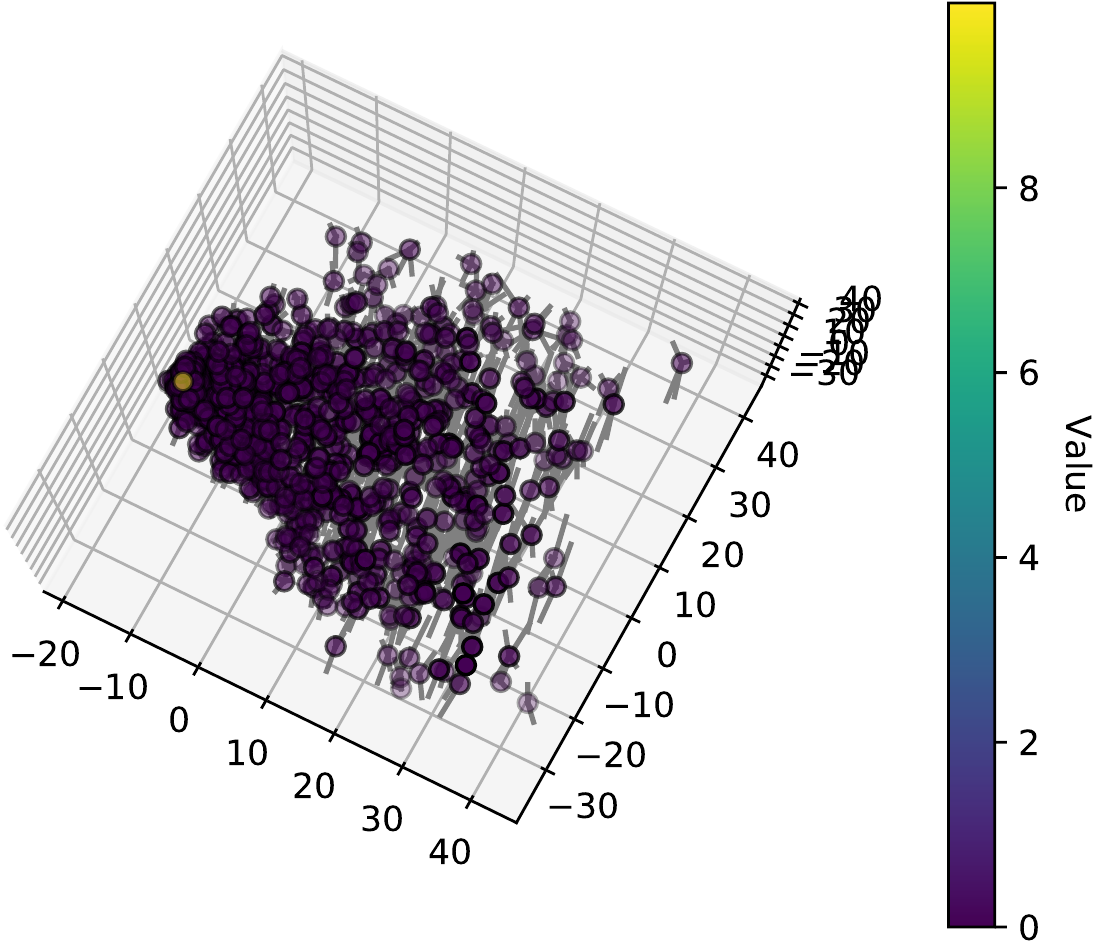}
	    \caption{WM-AE Baseline}
	    \label{fig:fmnist_wm}
	\end{subfigure}
	\begin{subfigure}{0.24\textwidth}
	    \centering
	    \includegraphics[width=\textwidth]{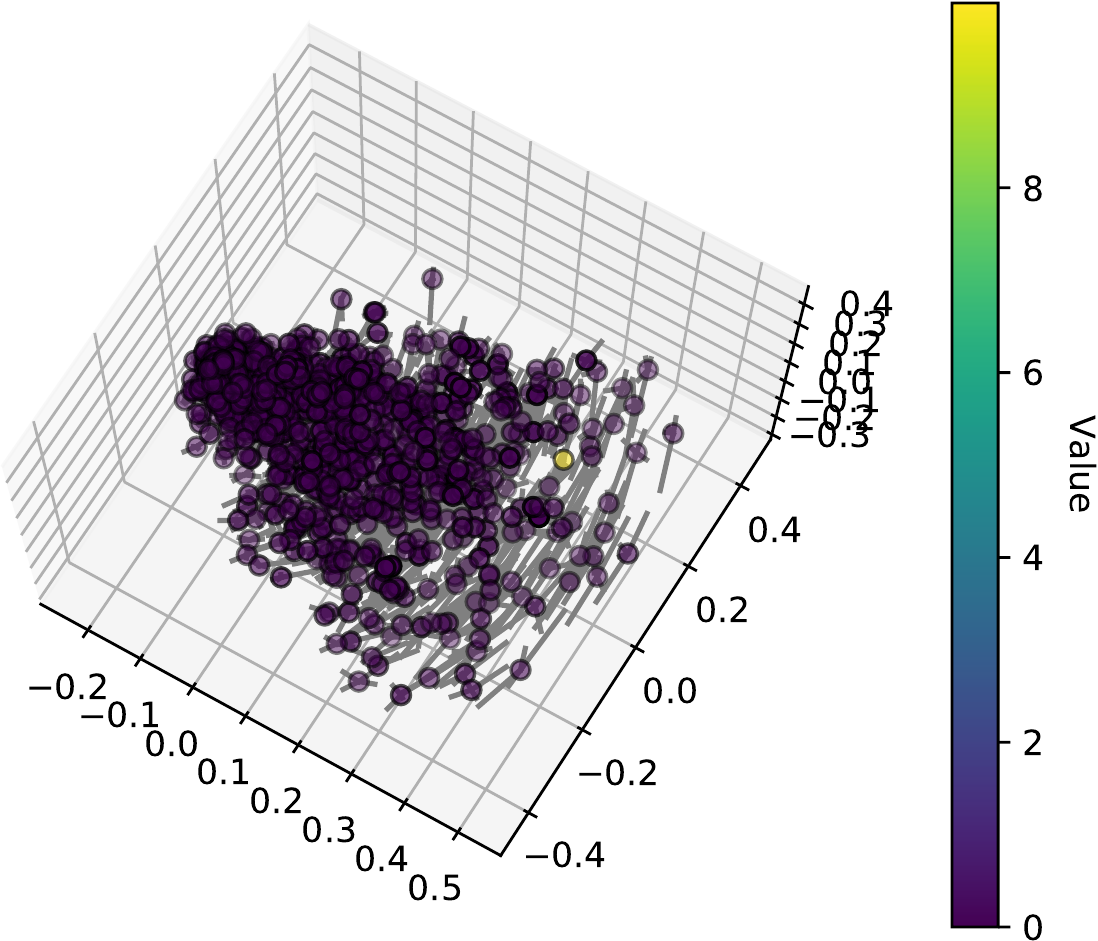}
	    \caption{LD-AE Baseline}
	    \label{fig:fmnist_recon}
	\end{subfigure}
	\begin{subfigure}{0.24\textwidth}
	    \centering
	    \includegraphics[width=\textwidth]{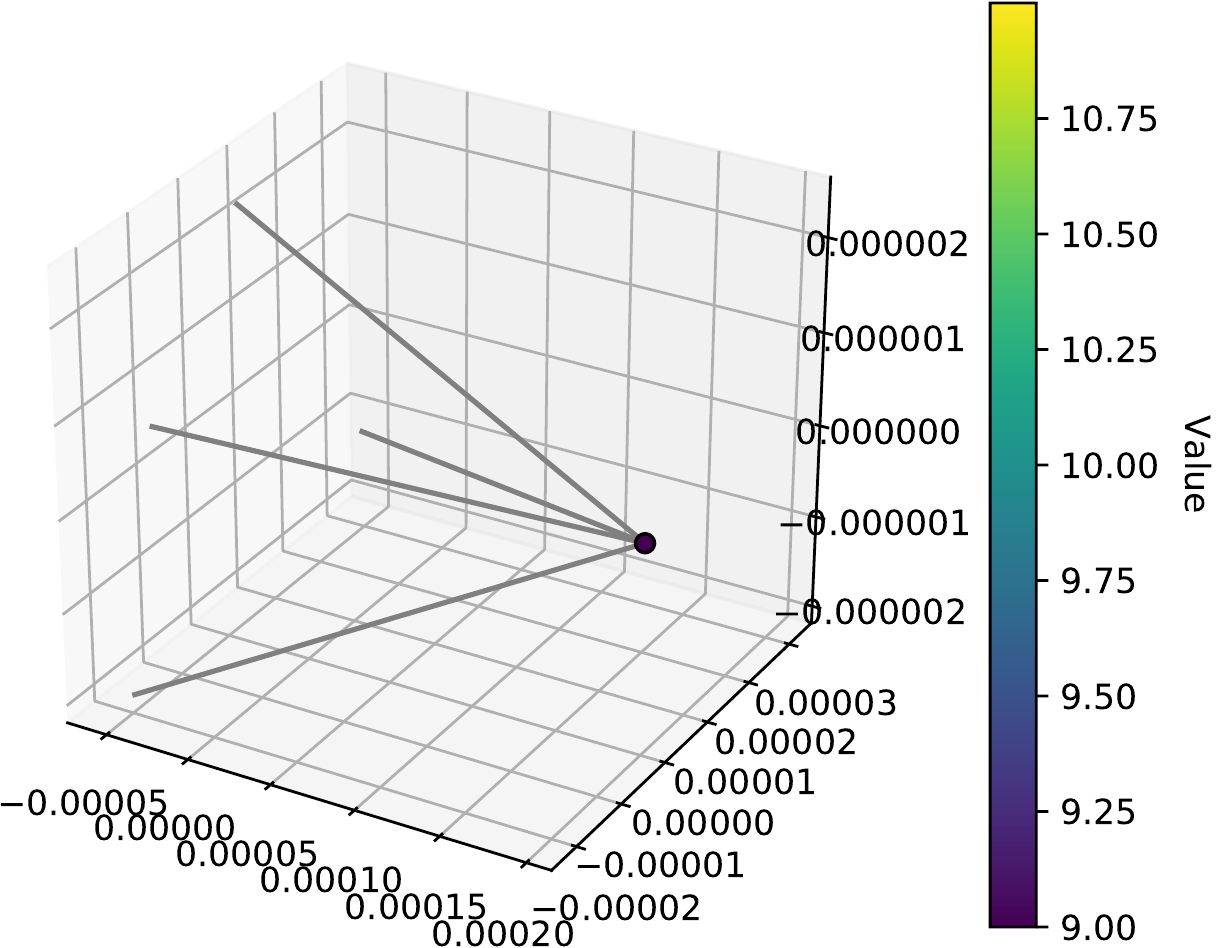}
	    \caption{DMDP-H Baseline}
	    \label{fig:fmnist_aceq0}
	\end{subfigure}
	\begin{subfigure}{0.24\textwidth}
	    \centering
	    \includegraphics[width=\textwidth]{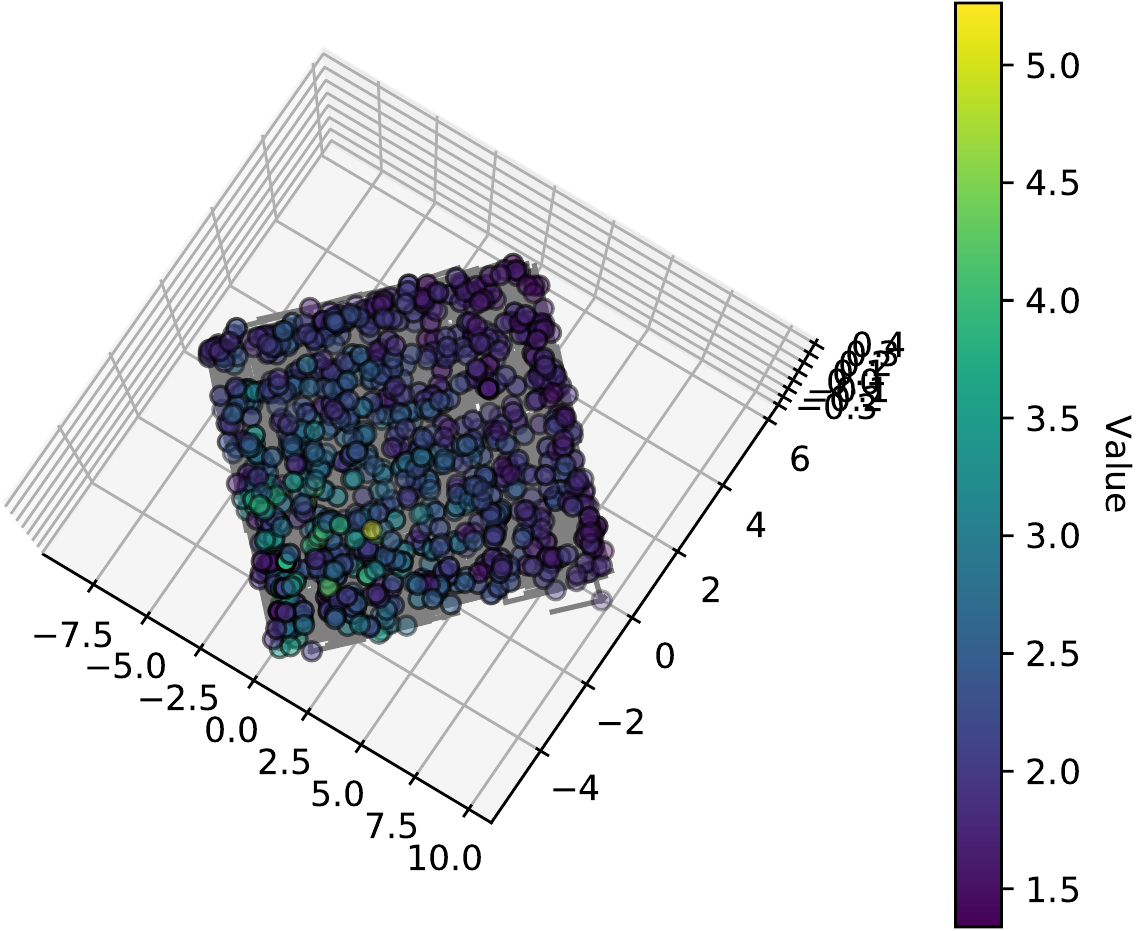}
	    \caption{This paper}
	    \label{fig:fmnist_aceq5}
	\end{subfigure}
	\caption{Abstract MDP for four approaches in planning in Fashion MNIST. Nodes are PCA projections of abstract states, edges are predicted $\atransparam$, colors are predicted values. }
	\label{fig:fmnist_latents}
\end{figure*}

\subsection{Generalizing over Goals and Objects}
In many tasks we need to be able to generalize not only over goals, but also object instances. We evaluate if our abstract state space generalizes to unseen objects in a problem class. For this we construct an object manipulation task. On each episode, an image of a piece of clothing is sampled from a set of training images in Fashion MNIST~\cite{xiao2017fashion}, and a goal translation of the image is sampled from a set of train goals (translations with negative $x$-offset: $(-3, \cdot)$ up to and including $(-1, \cdot)$). Thus, the underlying state space is a $7\times7$ grid. The translated image is provided to the agent as a goal state. The agent receives a reward of +1 if she moves the clothing to the correct translation. See Figure~\ref{fig:fmnist_overview}. At test time, we evaluate performance on test goals (translations with positive $x$-offset: $(1, \cdot)$ up to and including $(3, \cdot)$, seen before as states for training images but never as goals) and test images.
The latent spaces for each of the four representation learning approaches are shown in Figure~\ref{fig:fmnist_latents}. For DMDP-H, the latent space collapses to all but a few points. For WD-AE and LD-AE, the latent space does not exhibit clear structure. For our approach, there is a clear square grid structure present in the latent space. However, the underlying translations for the images do not neatly align across images. Clustering such states together is interesting future work.

\begin{wrapfigure}{r}{0.15\textwidth}
    \centering
    \includegraphics[width=0.15\textwidth]{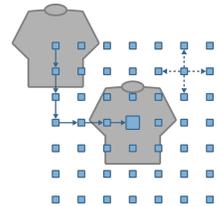}
	\caption{Transitions in the image manipulation task.}
    \label{fig:fmnist_overview}
	\vspace{-1em}
\end{wrapfigure}
Results are shown in Table~\ref{tab:fmnist}. The goal-conditioned model-free baseline has an easy time finding a good policy for the training setting. It also generalizes well to unseen images. However, it has trouble generalizing to new goal locations for both train and test images. Our planning approach, on the other hand, loses some performance on the training setting, but easily generalizes to both test images and test goals. Neither WM-AE nor LD-AE find good policies in this problem. They have a difficult time learning plannable representations because their focus is on reconstructing individual images.
\begin{table} 
\small
\begin{tabular}{lrrrr}
\toprule
	 \multicolumn{1}{c}{Avg. ep. length $\downarrow$} \\
Dataset & \multicolumn{2}{c}{\textbf{Train}} & \multicolumn{2}{c}{\textbf{Test}}\\
\emph{Goal Set} & Train & Test & Train & Test \\
\midrule
GC-Model-free & \textbf{4.82} \pmin{0.33} & 9.67 \pmin{5.01} & \textbf{4.75} \pmin{0.12} & 8.17 \pmin{2.67} \\
WM-AE & 59.95 \pmin{4.06} & 63.27 \pmin{3.36} & 64.27 \pmin{5.33} & 63.41 \pmin{2.04}  \\
LD-AE &  56.39 \pmin{7.07} & 49.35 \pmin{4.05} & 51.45 \pmin{6.79} & 51.70 \pmin{3.97}  \\
	DMDP-H ($J=0$) & 62.86 \pmin{3.87} & 66.68 \pmin{4.40} & 65.93 \pmin{4.98} & 64.86 \pmin{1.57} \\
\rowcolor{Gray}
Ours, $J=1$, & 5.07 \pmin{0.87} & \textbf{5.27} \pmin{0.56} & 5.69 \pmin{0.93} & 5.63 \pmin{0.96} \\
\rowcolor{Gray}
Ours, $J=3$ & 5.60 \pmin{0.97} & 5.46 \pmin{0.97} & 6.44 \pmin{1.12} & 5.42 \pmin{0.89} \\
\rowcolor{Gray}
Ours, $J=5$ & 5.36 \pmin{0.71} & 5.67 \pmin{1.20} & 6.36 \pmin{1.21} & \textbf{5.34} \pmin{0.93} \\
\bottomrule
\end{tabular}
	\caption{Comparing average episode length of 100 episodes for planning in Fashion MNIST. Reporting mean and standard deviation over 5 random seeds.}
\label{tab:fmnist}
\end{table}

%% file: related.tex
This paper proposes a method for learning action equivariant mappings of MDPs, and using these mappings for constructing plannable abstract MDPs. We learn action equivariant maps by minimizing MDP homomorphism metrics~\cite{taylor2009bounding}. As a result, when the loss reaches zero the learned mapping is an MDP homomorphism~\cite{ravindran2004approximate}. MDP homomorphism metrics are a generalization of bisimulation metrics~\cite{ferns2004metrics, li2006towards}. Other works~\cite{higgins2018towards, cohen2016group, worrall2017harmonic} consider equivariance to symmetry group actions in learning. Here, we use a more general version of equivariance under MDP actions for learning representations of MDPs.
We learn representations of MDPs by 1) predicting the next latent state, 2) predicting the reward and 3) using negative sampling to prevent state collapse. 
Much recent work has considered self-supervised representation learning for MDPs. Certain works focus on predicting the next state using a contrastive loss~\cite{cswm, anand2019unsupervised, oord2018representation}, disregarding the reward function. However, certain states may be erronously grouped together without a reward function to distinguish them. \citet{gelada2019deepmdp} include both rewards and transitions to propose an objective based on stochastic bisimulation metrics~\cite{givan2003equivalence, ferns2004metrics, li2006towards}. However, at training time they focus on deterministically predicting the next latent state. Their proposed objective does not account for the possibility of latent space collapse, and for complex tasks they require a pixel reconstruction term. This phenomenon is also observed by~\citet{franccois2018combined}, who prevent it with two entropy maximization losses. \\
Many approaches to representation learning in MDPs depend (partially) on learning to reconstruct the input state~\cite{corneil2018efficient, watter2015embed, ha2018world, igl2018deep, hafner2018learning, thomas2017independently, kurutach2018learning, zhang2018solar, kaiser2019model, watters2019cobra, wang2019learning, asai2019unsupervised}. A disadvantage of reconstruction losses is training a decoder, which is time consuming and usually not required for decision making tasks. Additionally, such losses emphasize visually salient features over features relevant to decision making. \\
Other approaches that side-step the pixel reconstruction loss include predicting which action caused the transition between two states~\cite{agrawal2016learning}, predicting the number of time steps between two states~\cite{aytar2018playing} or predicting objects in an MDP state using supervised learning~\cite{zhang2018composable}.\\
~\citet{jonschkowski2015learning} identify a set of priors about the world and uses them to formulate self-supervised objectives. In~\citet{ghosh2018learning}, the similarity between two states is the difference in goal-conditioned policies needed to reach them from another state. ~\citet{schrittwieser2019mastering} learn representations for tree-based search that must predict among others a policy and value function, and are thus not policy-independent.
Earlier work on decoupling representation learning and planning exists~\cite{corneil2018efficient, watter2015embed, zhang2018composable}. However, these works use objectives that include a pixel reconstruction term~\cite{corneil2018efficient, watter2015embed} or require labeling of objects in states for use in supervised learning~\cite{zhang2018composable}. \\
Other work on planning algorithms in deep learning either assumes knowledge of the state graph~\cite{tamar2016value, niu2017generalized, lee2018gated, karkus2017qmdp}, builds a graph out of observed transitions~\cite{klissarov2018diffusion} or structures the neural network architecture as a planner~\cite{oh2017value, farquhar2018treeqn, franccois2018combined}, which limits the search depth.

%% file: conclusion.tex
This paper proposes the use of `equivariance under actions' for learning representations in deterministic MDPs. Action equivariance is enforced by the use of MDP homomorphism metrics in defining a loss function. We also propose a method of constructing plannable abstract MDPs from continuous latent spaces. We prove that for deterministic MDPs, when our objective function is zero and our method for constructing abstract MDP is used, the map we learn is an MDP homomorphism. Additionally, we show empirically that our approach is data-efficient and fast to train, and generalizes well to new goal states and instances with the same environmental dynamics. Potential future work includes an extension to stochastic MDPs and clustering states on the basis of MDP metrics. Using a clustering approach as part of model training, we can learn the prototypical states rather than sampling them. This comes at the cost of having to backpropagate through a discretization step, which in early experiments (using Gumbel-Softmax~\cite{jang2016categorical}) led to instability.

%% file: acknowledgments.tex
We thank Patrick Forr\'e, Jorn Peters and Michael Herman for helpful comments.
T.K. acknowledges funding by SAP SE.
\newlength{\tmplength}
\setlength{\tmplength}{\columnsep}
F.A.O. received funding from the European Research Council (ERC)
\begin{wrapfigure}{r}{0.2\columnwidth}
    \vspace{-5pt}
    \includegraphics[width=0.2\columnwidth]{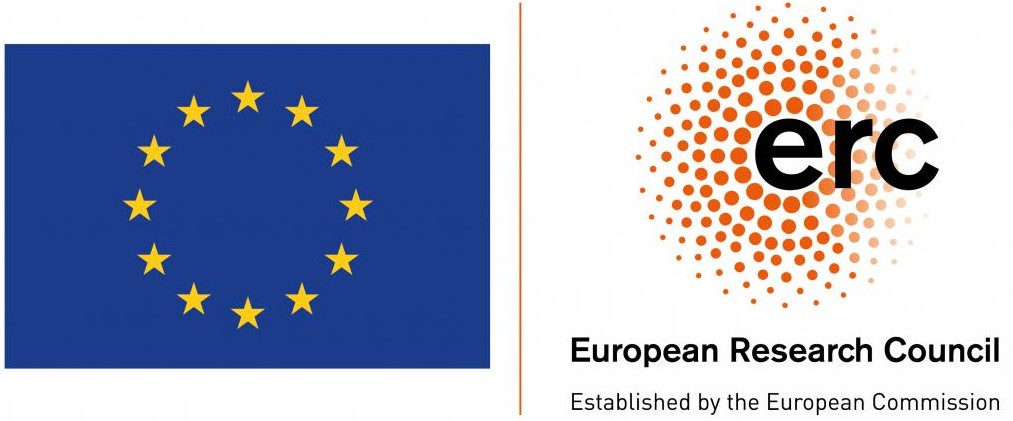}
\end{wrapfigure}
under the European Union's
Horizon 2020 research
and innovation programme (grant agreement No.~758824 \textemdash INFLUENCE).
\setlength{\columnsep}{\tmplength}